\documentclass{article} 
\usepackage{iclr2027_conference,times}

\usepackage{amsmath,amsfonts,bm}

\newcommand{\captiona}{{\em (a)}}
\newcommand{\captionb}{{\em (b)}}
\newcommand{\captionc}{{\em (c)}}
\newcommand{\captiond}{{\em (d)}}

\def\eqref#1{equation~\ref{#1}}

\def\1{\bm{1}}

\DeclareMathAlphabet{\mathsfit}{\encodingdefault}{\sfdefault}{m}{sl}
\SetMathAlphabet{\mathsfit}{bold}{\encodingdefault}{\sfdefault}{bx}{n}

\usepackage{hyperref}
\usepackage{url}
\usepackage{graphicx}
\usepackage{xcolor}
\usepackage{colortbl}
\usepackage{booktabs}
\usepackage{algorithm}
\usepackage{algorithmic}
\usepackage{amsthm}
\usepackage{pifont}
\newtheorem{assumption}{Assumption}
\newtheorem{proposition}{Proposition}
\newtheorem{corollary}{Corollary}
\newcommand{\cmark}{\ding{51}}
\newcommand{\xmark}{\ding{55}}
\usepackage{listings}
\definecolor{codecomment}{rgb}{0.25,0.5,0.5}
\definecolor{codefunc}{rgb}{0.85,0.18,0.50}
\definecolor{codesign}{RGB}{0,0,180}
\lstdefinestyle{demodpo}{
  language=Python,
  basicstyle=\fontencoding{T1}\fontfamily{lmtt}\fontsize{9pt}{9.9pt}\selectfont,
  keywordstyle=\color{blue}\bfseries,
  commentstyle=\color{codecomment},
  stringstyle=\color{orange},
  emph={fm_mse,stopgrad,sample,rand,randn_like,mean,log_sigmoid},
  emphstyle=\color{codefunc},
  frame=none, backgroundcolor=\color{white},
  columns=fullflexible, keepspaces=true,
  showstringspaces=false,
  breaklines=true, breakatwhitespace=true,
  aboveskip=0.6em, belowskip=0.4em,
  literate=
    {*}{{{\color{codesign}*}}}{1}
    {-}{{{\color{codesign}-}}}{1}
    {+}{{{\color{codesign}+}}}{1},
}
\providecommand{\UseTaggingSocket}[1]{}
\providecommand{\SuspendTagging}[1]{}
\providecommand{\ResumeTagging}[1]{}

\title{LeWAM: Latent evolving World Action Model}

\author{%
Xueji Fang$^{1,2,3}$\quad
Boqiang Duan$^{3}$\quad
Hua Wu$^{3}$\quad
\textbf{Jingdong Wang}$^{3,\dagger,\ddagger}$\quad
\textbf{Guo-Jun Qi}$^{2,\dagger}$
}

\iclrfinalcopy 
\begin{document}

\maketitle

\ificlrfinal
\begin{center}
    \vspace{-2.5em}
    \small
    $^{1}$Zhejiang University \hspace{2.5em}
    $^{2}$Westlake University \hspace{2.5em}
    $^{3}$Baidu Inc. \\
\end{center}
\begingroup
\renewcommand{\thefootnote}{\ensuremath{\dagger}}
\begin{NoHyper}
\footnotetext{Corresponding authors \quad $^{\ddagger}$Project leader}
\end{NoHyper}
\endgroup
\fi

\begin{abstract}

World Action Models (WAMs) jointly model action generation and environment dynamics and are mostly built on pretrained Video Diffusion Models (VDMs). 
In VDM-based WAMs, observations are first encoded by a VAE, and the resulting compressed latents are then processed by large video diffusion backbones to extract effective features for action generation.  
However, this paradigm ties WAM performance and training cost to large-scale video generation pretraining, limiting WAM efficiency and scalability.
In this paper, we theoretically and empirically investigate how visual representations affect action generation in WAMs.
Our results show that predictive embeddings from Joint-Embedding Predictive Architecture (JEPA) encoders better support action generation than compressed VAE latents, with I-JEPA performing best in our encoder comparison.
Based on these findings, we propose LeWAM, which conditions action generation on JEPA embeddings and models environment evolution by predicting future embeddings in the same space, without relying on a video diffusion backbone.
We further find that imitation learning matches demonstrated actions but does not distinguish better actions from worse ones, even though small action deviations can greatly affect task success.
To address this limitation without additional environment interaction or the human oversight required for resets and safety, we introduce Demonstration-Guided DPO (DemoDPO), an offline preference refinement stage that derives preference supervision directly from demonstrations.
With only 0.4B trainable parameters, LeWAM achieves an average success rate of 92.28\% on RoboTwin 2.0, comparable to that of state-of-the-art VLAs and WAMs, and maintains practical effectiveness on real-world manipulation tasks.

\end{abstract}

\ificlrfinal
\begingroup
\hypersetup{pdfborder={0 0 0}}
\begin{center}
    \vspace{-1.0em}
    \href{https://github.com/XuejiFang/LeWAM}{%
        \raisebox{-0.18\height}{\includegraphics[height=1.1em]{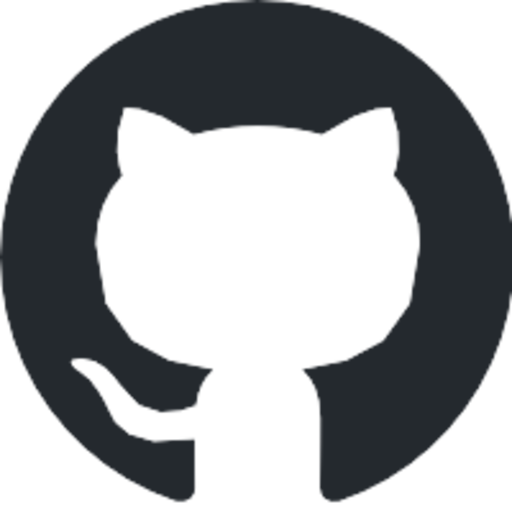}}%
        \hspace{0.3em}\small\textcolor[HTML]{1777BC}{\texttt{github.com/XuejiFang/LeWAM}}%
    }
    \hspace{2em}
    \href{https://huggingface.co/XuejiFang/LeWAM}{%
        \raisebox{-0.18\height}{\includegraphics[height=1.1em]{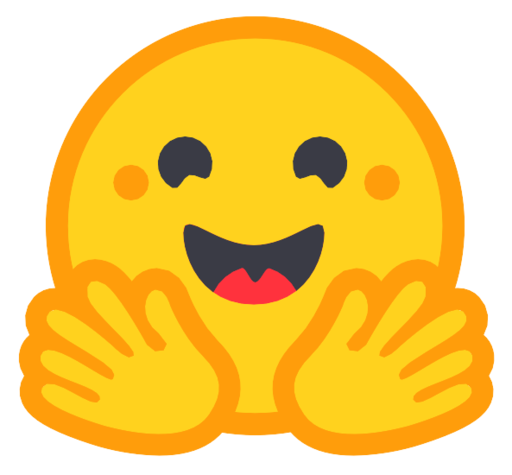}}%
        \hspace{0.3em}\small\textcolor[HTML]{1777BC}{\texttt{huggingface.co/XuejiFang/LeWAM}}%
    }
    \vspace{-0.2em}
\end{center}
\endgroup
\fi

\begin{figure}[!htbp]
    \centering
    \includegraphics[width=\linewidth]{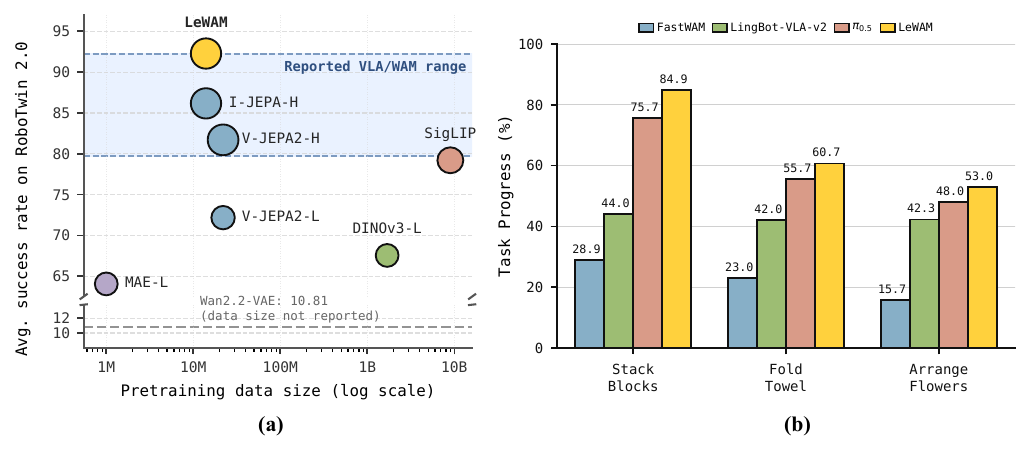}
    \vspace{-2em}
    \caption{
      \captiona~\textbf{Frozen vision encoder comparison and LeWAM performance on RoboTwin 2.0.} Frozen encoders are evaluated with the same simple action predictor architecture and training recipe. I-JEPA-H, pretrained only on ImageNet-22K, performs best among the evaluated encoders, while LeWAM achieves an average success rate of 92.28\%.
     \captionb~\textbf{Real-world task progress across three manipulation tasks.} LeWAM achieves the highest task progress on all three tasks, outperforming $\pi_{0.5}$, LingBot-VLA-v2, and FastWAM.
    }
    \label{fig:main_results}
\end{figure}

\section{Introduction}
\label{sec:introduction}

Vision-Language-Action (VLA) models~\citep{RT-1, RT-H, openvla, llada_vla, long_vla, pi_05, CombatVLA, RTC, giga_brain, traj2act} have become a widely used paradigm for generalist robotic policies, mapping visual observations and language instructions directly to actions via behavioral cloning. By leveraging large pretrained vision-language backbones, VLAs generalize across language instructions, but they do not explicitly model how the physical world evolves under actions. World Action Models (WAMs)~\citep{DreamZero, gigaworld-policy, dit4dit, fast-wam} address this limitation by coupling action prediction with visual dynamics learning, often using Video Diffusion Model (VDM) backbones to model how actions change the physical scene. These WAMs provide additional supervision for embodied policies and report competitive results on manipulation tasks. 

Despite these advances, recent WAMs reveal a tension between how they are trained and deployed. Some WAMs retain explicit joint video and action generation at inference time~\citep{DreamZero, motus, dit4dit}, but action-centered WAMs increasingly remove future visual generation during deployment~\citep{gigaworld-policy, fast-wam}. These WAMs still benefit from video supervision during training, but at test time often generate actions directly from the current observation and task instruction. At deployment, they use VLA-style inference, predicting actions directly without forecasting future state changes. For VDM-based WAMs, the visual representation can therefore become a bottleneck for action generation: compressed VAE posterior means can be poorly suited to lightweight action predictors, and extracting useful features from them relies on a large pretrained video diffusion backbone, limiting efficiency and scalability.

Prior studies provide evidence on visual representations for action generation, but their settings are not directly comparable. 
Existing encoder comparisons cover a limited set of representations or evaluate them as additions to VLA features~\citep{GR-1, dinov3_diffusion_policy, vpp, jepa_vla, DynaMind}. 
Existing latent world models use predictive embeddings for goal-conditioned planning rather than direct action generation~\citep{vjepa2, le_world_model}. 
Existing VLAs and VDM-based WAMs rely on different visual pipelines, using pretrained VLM features and VAE latents processed by video diffusion backbones, respectively~\citep{openvla, pi_0, pi_05, llada_vla, x-vla, DreamZero, motus, dit4dit, fast-wam, gigaworld-policy}. 
None of these studies compares representative frozen encoders learned through different pretraining objectives as direct inputs to the same action predictor.
These objectives include reconstruction~\citep{mae,cae}, self-supervised feature learning~\citep{dino,Aet,Avt}, vision-language alignment~\citep{siglip}, and image or video prediction~\citep{ijepa,vjepa2,CAEv2}.
Such a comparison is necessary to isolate the contribution of the visual representation from the policy and training procedure.

We therefore analyze WAMs built on VDMs from the perspective of information theory and compare representative frozen encoders under the same training recipe. Our analysis separates the action information a representation retains from the part a downstream model can use, while the controlled comparison shows that compressed VAE latents are poorly suited to direct action generation and that I-JEPA performs best among the evaluated encoders.
These findings motivate LeWAM, a \textbf{L}atent \textbf{e}volving \textbf{W}orld \textbf{A}ction \textbf{M}odel that performs both action generation and world modeling in JEPA embedding space. To jointly learn action generation and future embedding prediction conditioned on actions, LeWAM organizes the current embedding, noisy and demonstrated action tokens, and future embedding queries in a unified token sequence and controls the information flow with a structured attention mask. This design allows both objectives to be learned by a single predictor without information leakage. To further exploit the information distributed across the frozen encoder, we introduce AdaFuse, which adaptively combines  representations from different encoder layers for each feature dimension. At inference time, LeWAM removes the demonstrated actions and future queries and generates actions only from the current observation and task context.

\begin{figure}[htpb]
    \centering
    \includegraphics[width=\linewidth]{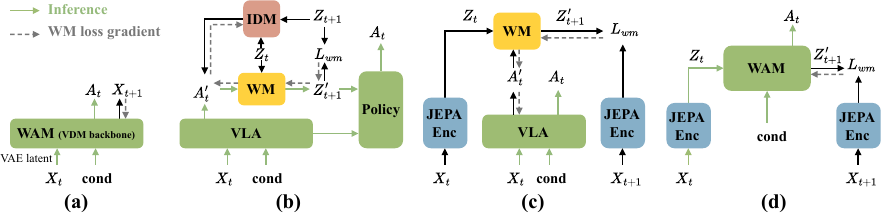}
    \vspace{-1.5em}
      \caption{
          Comparison of WAM architectures.
        \textbf{(a) Action-centered VDM-based WAMs} use compressed VAE latents as visual inputs and jointly learn action generation and future frame prediction, but generate only actions at inference.
        \textbf{(b) LaWAM}~\citep{lawam} predicts latent actions from VLM representations and uses a separate world model to predict future DINOv3 embeddings, which condition action generation at inference.
        \textbf{(c) VLA-JEPA}~\citep{vla_jepa} generates actions from VLM representations and uses latent action tokens to condition a separate world model supervised by V-JEPA2 embeddings during training.
        \textbf{(d) LeWAM} ses frozen I-JEPA embeddings as both visual inputs for action generation and targets for action-conditioned future prediction, learning both tasks with a single predictor while generating only actions at inference.
          }
    \label{fig:wam_architectures}
\end{figure}

We observe that LeWAM generates actions close to demonstrations on the validation set, yet small deviations during interaction with the environment can cause task failure. We further find that allowing the policy multiple attempts increases its success rate. This suggests that successful action trajectories lie within the policy distribution but receive insufficient probability mass.
We therefore aim to shift the policy distribution toward these trajectories. A natural approach is to maximize expected return through online reinforcement learning with rollout rewards~\citep{dppo, flow_policy_gradients}.
However, on physical robots, this requires repeated environment interaction and human oversight for resets and safety, making it difficult to scale~\citep{leave_no_trace, autonomous_rl, ariel, risc}. To avoid this requirement, we introduce Demonstration-Guided DPO (DemoDPO) as the offline preference refinement stage of LeWAM.
DemoDPO uses demonstrations to construct preferences among action candidates from a frozen reference policy, providing offline preference supervision without additional environment interaction or manually designed rewards.

Our contributions are summarized as follows:
\begin{itemize}
    \item We theoretically and empirically study visual representations for action generation in WAMs, identify I-JEPA as an effective visual encoder for action generation, and further examine its representations through attention visualization and probing.
    
    \item We propose LeWAM, which combines AdaFuse with joint action generation and future embedding prediction in a lightweight model, while generating only actions at inference.

    \item We refine LeWAM with DemoDPO, using demonstrations to provide preference supervision without additional environment interaction or manually designed rewards.
\end{itemize}

\section{Related Work}
\label{sec:related_work}

\textbf{Visual Representations and World Modeling for Action Generation.}
Diffusion Policy~\citep{diffusion_policy} and Octo~\citep{octo} generate action sequences through conditional denoising without an explicit world model.
VDM-based WAMs typically couple action learning with future video modeling through pretrained video diffusion backbones, as illustrated in Figure~\ref{fig:wam_architectures}\captiona{}~\citep{DreamZero, motus, dit4dit, fast-wam, gigaworld-policy}.
Recent methods instead predict future states in pretrained representation spaces.
VLA-JEPA uses latent action tokens produced by a VLM to condition a separate world model trained against frozen V-JEPA2 targets~\citep{vla_jepa}.
The flow matching action head is conditioned on VLM representations, whereas V-JEPA2 embeddings supervise only dynamics learning and do not serve as direct visual inputs for action generation, as shown in Figure~\ref{fig:wam_architectures}\captionc{}.
LaWAM first learns a latent action model in a frozen DINOv3 space and then factorizes control into a VLA that predicts a latent action, a separate latent world model that decodes it into a future visual subgoal, and an action expert conditioned on that subgoal~\citep{lawam}.
Future prediction therefore remains an explicit intermediate at deployment, as shown in Figure~\ref{fig:wam_architectures}\captionb{}.
LeWAM instead uses frozen I-JEPA embeddings as both the direct visual input for action generation and the target for future prediction.
A single predictor jointly generates actions and predicts future embeddings conditioned on demonstrated action prefixes during training, while future prediction is removed at deployment, as shown in Figure~\ref{fig:wam_architectures}\captiond{}.

\textbf{Robot Policy Refinement beyond Imitation.}
Online reinforcement learning has been used to refine diffusion and flow matching policies using rewards collected through environment interaction~\citep{dppo, flow_policy_gradients}.
However, applying these methods to physical robots requires repeated environment rollouts, reliable task rewards, and human oversight for resets and safety, making them difficult to scale across tasks~\citep{leave_no_trace, autonomous_rl, ariel, risc}.
Preference-based methods provide an alternative, but existing approaches derive supervision from predefined task rewards or reward-based evaluators~\citep{pao_dp, fkpd}, model-generated evaluation criteria~\citep{grape}, or direct human feedback, including pairwise comparisons, interventions, and manually specified behavior preferences~\citep{fdpp, apo, rodif}.
These signals still require additional supervision or interaction beyond the original offline demonstrations.
DemoDPO instead uses each existing demonstration to rank a group of candidates sampled from a frozen reference policy and filters groups without sufficiently distinct preferences.
It therefore enables offline preference refinement without additional environment interaction, reward supervision, or human preference annotation.
\section{Representation-Guided Latent World Action Modeling}
\label{sec:method}

In this section, we analyze visual representations for action generation and the limitations of VDM-based WAMs.
We then present LeWAM and its DemoDPO preference refinement stage.

\subsection{Problem Formulation}
\label{ssec:problem_formulation}

We study visuomotor policy learning from expert demonstrations.
At time $t$, the policy maps a visual observation $X_t$ and task context $c_t$ to an expert action chunk
$A_t=(a_t,\ldots,a_{t+H_{\mathrm{a}}-1})$, where $H_{\mathrm{a}}$ is the number of control steps per generation pass.
LeWAM parameterizes the conditional policy $\pi_\theta(A_t\mid X_t,c_t)$ using flow matching \cite{FlowMatching}.
World Action Models augment this objective by modeling environment evolution over the same horizon.
We select $N$ future offsets $0<\delta_1<\cdots<\delta_N\leq H_{\mathrm{a}}$ and define
$Z_i=E(X_{t+\delta_i})$, where the visual representation function $E$ produces
$Z_i\in\mathbb{R}^{hw\times d}$ as an $h\times w$ grid of $d$-dimensional patch embeddings.
LeWAM jointly generates $A_t$ and predicts $\{Z_i\}_{i=1}^N$, as detailed in Section~\ref{ssec:lewam_design}.
Below, we omit time indices and let $X$, $A$, $c$, and $Z=E(X)$ denote random variables drawn from the demonstrations.
Unless stated otherwise, all probabilities, expectations, and information quantities refer to their joint distribution.

\subsection{Action-Sufficient World Modeling}
\label{ssec:action_sufficient_latents}

We analyze how a visual representation affects action uncertainty.
\begin{assumption}[Action sufficiency]\label{as:sufficiency}
The factors $S=S(X)$ form a sufficient statistic for the action given the context, that is, $p(A\mid X,c)=p(A\mid S,c)$.
\end{assumption}

\begin{proposition}[Action relevance gap]\label{prop:gap}
Under Assumption~\ref{as:sufficiency}, let $Z=E(X)$, where $E$ is deterministic.
When the conditional action entropies are finite, the additional action uncertainty satisfies
\begin{equation}
    \Delta_E(A)
    :=
    H(A\mid Z,c)-H(A\mid X,c)
    =
    I(A;X\mid Z,c)
    =
    I(A;S\mid Z,c)
    \geq 0,
    \label{eq:action_factor_gap}
\end{equation}
where $H$ denotes entropy for discrete actions and differential entropy for continuous actions, and $I$ denotes conditional mutual information.
\end{proposition}

A useful representation for action generation should retain the information in $S$ relevant to $A$.

We next examine VAE compression and the reliance on large video diffusion backbones in VDM-based WAMs.
First, a linear Gaussian analysis makes explicit how VAE training can suppress observation directions in the posterior mean~\citep{wang2022posterior}.

\begin{proposition}[VAE spectral compression]\label{prop:vae}
For fixed $c$, consider the linear Gaussian VAE in Appendix~\ref{app:vae_spectral_compression}, optimized to a global minimum with latent width $k$, KL weight $\beta_{\mathrm{vae}}>0$, and fixed decoder variance $\sigma_{\mathrm{dec}}^2>0$.
Let $(\lambda_i,u_i)$ denote the observation covariance eigenpairs in decreasing eigenvalue order, and set $\lambda_{\mathrm{cut}}=\beta_{\mathrm{vae}}\sigma_{\mathrm{dec}}^2$.
The posterior mean $Z_{\mathrm{vae}}$ has effective dimension
\begin{equation}
    r_{\mathrm{eff}}
    :=
    \operatorname{rank}\operatorname{Cov}(Z_{\mathrm{vae}}\mid c)
    =
    \min\bigl(k,\#\{i:\lambda_i>\lambda_{\mathrm{cut}}\}\bigr).
    \label{eq:method_vae_effective_dimension}
\end{equation}
Let $J=\{i\leq k:\lambda_i>\lambda_{\mathrm{cut}}\}$.
If $A=B_cX+\xi$, where $\xi$ is independent Gaussian noise with positive definite covariance, then $\Delta_E(A)>0$ exactly when $B_cu_i\neq0$ for some $i\notin J$.
\end{proposition}

Proposition~\ref{prop:vae} and Corollary~\ref{cor:misalign} show that reconstruction based selection need not retain action-relevant directions.
Appendix~\ref{app:usable_action_uncertainty} further distinguishes information loss from the difficulty of extracting retained information with a fixed predictor~\citep{xu2020usable}.
This distinction motivates our controlled comparison of frozen representations for action generation in Section~\ref{ssec:frozen_encoder_comparison}.

Second, VDM-based WAMs rely on a large video diffusion backbone pretrained at scale for useful visual features. Indeed, \citet{wang2026genception} show that such a backbone can be repurposed into a strong feedforward perception model. This result demonstrates the value of video generation pretraining, but does not establish the effectiveness of raw VAE latents as direct representations for action generation. The diffusion backbone extracts the information retained in $Z_{\mathrm{vae}}$ and exploits learned visual priors.
Section~\ref{ssec:frozen_encoder_comparison} evaluates whether a lightweight predictor on a frozen JEPA encoder matches these systems without such a backbone.

\subsection{World Action Modeling in Representation Space}
\label{ssec:lewam_design}
\begin{figure}[htpb]
    \centering
    \includegraphics[width=\linewidth]{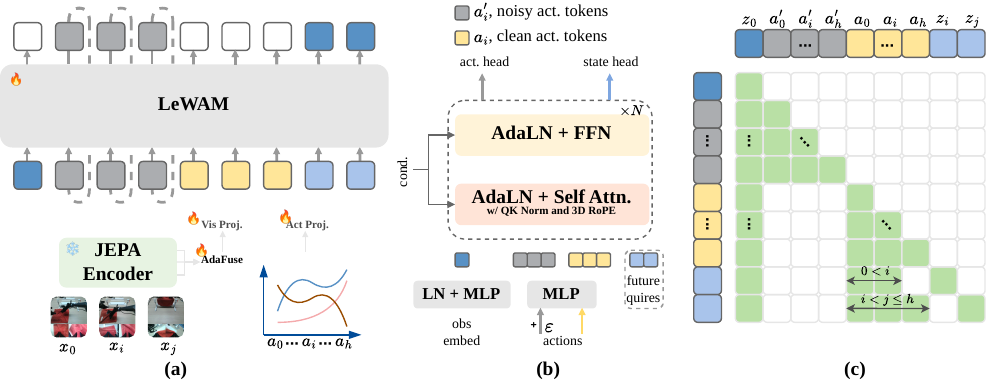}
    \vspace{-1.5em}
    \caption{
      LeWAM architecture for joint action generation and future embedding prediction.
      \captiona~A frozen JEPA encoder extracts multilayer embeddings, and AdaFuse adaptively combines them.
      \captionb~The predictor combines observation and action tokens with future queries to predict action velocities and future embeddings.
      \captionc~The mask prevents clean actions from leaking into action prediction and restricts each future query to the action prefix for its randomly sampled horizon.
  }
    \label{fig:method_overview}
\end{figure}

Figure~\ref{fig:method_overview} shows how LeWAM learns an action policy and a representation space world model within a single predictor:
\begin{equation}
    \pi_\theta(A_t\mid X_t,c_t),
    \qquad
    p_\theta\!\left(Z_i\mid X_t,A_t^{(\delta_i)},c_t\right),
    \quad
    i=1,\ldots,N,
\end{equation}
where $A_t^{(\delta_i)}=(a_t,\ldots,a_{t+\delta_i-1})$ is the clean action prefix leading to $X_{t+\delta_i}$ and $Z_i=E(X_{t+\delta_i})$.
The first conditional generates an action chunk from the current observation, while the second predicts how the visual embedding evolves under the demonstrated action prefix.
Given a training segment starting at time $t$, let $E_{\ell}(X)\in\mathbb{R}^{hw\times d}$ denote the patch embeddings from layer $\ell$ of a frozen JEPA encoder with $L$ layers.
AdaFuse learns logits $\Gamma\in\mathbb{R}^{L\times d}$, normalizes them across layers separately for each feature dimension, and computes
\begin{equation}
    F_{\mathrm{Ada}}(X)
    =
    \operatorname{LN}_{\mathrm{out}}
    \left(
        \sum_{\ell=1}^{L}
        w_{\ell}\odot
        \operatorname{LN}_{\mathrm{in}}\!\left(E_{\ell}(X)\right)
    \right),
    \qquad
    w_{\ell,k}
    =
    \frac{\exp(\Gamma_{\ell,k})}
    {\sum_{r=1}^{L}\exp(\Gamma_{r,k})},
\end{equation}
where $\odot$ denotes elementwise multiplication, and $\operatorname{LN}_{\mathrm{in}}$ and $\operatorname{LN}_{\mathrm{out}}$ apply LayerNorm before and after fusion.
The channelwise fusion weights are shared across patch tokens and observations.
We use AdaFuse for the current observation and obtain future targets directly from the frozen encoder: $Z_0=F_{\mathrm{Ada}}(X_t)$ and $Z_i=E(X_{t+\delta_i})$, where $E$ denotes the frozen encoder.
For conditional flow matching, we sample $\tau\sim\operatorname{Unif}[0,1]$ and form
\begin{equation}
    A'_{\tau}=(1-\tau)A_t+\tau\epsilon,
    \qquad
    \epsilon\sim\mathcal{N}(0,\mathbf{I}),
    \qquad
    v_{\tau}=\epsilon-A_t.
    \label{eq:action_interpolation}
\end{equation}

To jointly model action generation and future embedding prediction, LeWAM processes the token sequence $[Z_0,A'_{\tau},A_t,Q_{1:N}]$ with a single predictor.
Each $Q_i$ is formed by repeating a learned query token over a patch grid matching $Z_0$ and adding learned position embeddings.
Each attention layer uses QK normalization and 3D RoPE with one temporal and two spatial coordinates.
$Z_0$ and $Q_i$ use their spatiotemporal positions, whereas noisy and clean action tokens use their temporal positions with both spatial coordinates set to $-1$.
A structured attention mask allows each noisy action token to attend only to $Z_0$ and its causal noisy action prefix, makes the clean action tokens causal, and allows $Q_i$ to attend only to $Z_0$, its own query grid, and the first $\delta_i$ clean actions.
This prevents clean actions from leaking into action generation and actions after $\delta_i$ from affecting the prediction of $Z_i$, either directly or through earlier clean action states.
The causal noisy action stream also supports different action horizons at inference.
Under this mask, the noisy action stream realizes $\pi_\theta(A_t\mid X_t,c_t)$, while the future query stream realizes $p_\theta(Z_i\mid X_t,A_t^{(\delta_i)},c_t)$.

Separate heads map the hidden states of the noisy action tokens to the flow velocity $\hat{v}_{\theta}$ and those of the future query tokens to the predicted future embeddings $\hat{Z}_{1:N}$.
We optimize both objectives with
\begin{equation}
    \mathcal{L}_{\mathrm{act}}
    =
    \mathbb{E}\!\left[\left\|\hat{v}_{\theta}-v_{\tau}\right\|_2^2\right],
    \qquad
    \mathcal{L}_{\mathrm{wm}}
    =
    \mathbb{E}_{i}\!\left[
    \left\|\hat{Z}_{i}-\mathrm{sg}(Z_{i})\right\|_2^2
    \right],
    \label{eq:lewam_losses}
\end{equation}
where $\mathbb{E}_{i}$ averages over the $N$ future targets and $\mathrm{sg}$ denotes stop gradient.
The joint objective is $\mathcal{L}_{\mathrm{LeWAM}}=\mathcal{L}_{\mathrm{act}}+\alpha\mathcal{L}_{\mathrm{wm}}$, where $\alpha$ controls the contribution of future embedding prediction.

\paragraph{Inference.}
At inference, LeWAM encodes $X_t$ into $Z_0$ and denoises an action chunk of the desired length conditioned on $Z_0$ and $c_t$.
Clean action tokens and future queries are omitted, so future embedding prediction provides training supervision without additional inference cost.

\subsection{Demonstration-Guided DPO}
\label{ssec:demodpo}

We observe that LeWAM can generate action chunks that closely match successful trajectories, yet small action deviations can still cause a rollout to fail.
This exposes a limitation of supervised imitation learning: flow matching learns the demonstrated action distribution but provides no explicit signal that distinguishes better generated actions from worse ones.
We therefore refine LeWAM with Demonstration-Guided DPO (DemoDPO), which uses each demonstrated action chunk to rank a group of candidates sampled from a frozen reference policy.
We initialize the trainable policy $\pi_\theta$ and a frozen reference policy $\pi_{\mathrm{ref}}$ from the supervised LeWAM checkpoint.
For each training condition $x_t=(X_t,c_t)$ and demonstrated action chunk $A_t$, we draw $G$ candidates $\widetilde{A}_i\overset{\mathrm{i.i.d.}}{\sim}\pi_{\mathrm{ref}}(\cdot\mid x_t)$ from independent Gaussian noise.
We measure their similarity to the demonstration using action MSE, denoted by $d_i$, and define the candidates with the smallest and largest MSE as $A^w$ and $A^l$, respectively.
The demonstration therefore provides the ranking signal, while both candidates are sampled from the same frozen reference policy.
Let $i^w$ and $i^l$ be the indices of $A^w$ and $A^l$, respectively, and let $\varepsilon_{\mathrm{pair}}$ denote the minimum MSE gap.
We define the pair retention mask as
\begin{equation}
    M_{\mathrm{pair}}
    =
    \mathbf{1}
    \left\{
    d_{i^l}-d_{i^w}\geq\varepsilon_{\mathrm{pair}}
    \right\}.
    \label{eq:demodpo_pair_mask}
\end{equation}
We retain only pairs with $M_{\mathrm{pair}}=1$, filtering out groups whose best and worst candidates have nearly indistinguishable demonstration errors.

For a preference pair $(A^w,A^l)$, standard DPO~\citep{dpo} minimizes
\begin{equation}
    \mathcal{L}_{\mathrm{DPO}}
    =
    -\mathbb{E}
    \left[
    \log\sigma
    \left(
    \beta
    \left[
    \rho_\theta(A^w\mid x_t)
    -
    \rho_\theta(A^l\mid x_t)
    \right]
    \right)
    \right],
    \qquad
    \rho_\theta(A\mid x_t)
    =
    \log\frac{\pi_\theta(A\mid x_t)}{\pi_{\mathrm{ref}}(A\mid x_t)},
    \label{eq:standard_dpo}
\end{equation}
where the expectation is over preference pairs, $\sigma$ is the sigmoid function, $\beta$ is the DPO regularization coefficient, and $\rho_\theta(A\mid x_t)$ is the log policy ratio of the trainable policy to the frozen reference policy.
A larger $\rho_\theta(A\mid x_t)$ means that the trainable policy assigns higher likelihood to $A$ relative to the reference policy.
For a flow matching policy, $\rho_\theta$ cannot be evaluated directly from the velocity predictor.
Diffusion-DPO~\citep{diffusion_dpo} derives a tractable surrogate objective for diffusion models from denoising errors at a sampled noise level, while Flow-DPO~\citep{flowdpo} adapts this objective to flow matching policies using velocity prediction errors.
We follow this construction by sampling $(\tau,\epsilon)$ independently of the candidate generation noises $\{\eta_i\}_{i=1}^{G}$ and applying Equation~\ref{eq:action_interpolation} to $A^{w}$ and $A^{l}$ with the same $(\tau,\epsilon)$.
For $\phi\in\{\theta,\mathrm{ref}\}$ and $A\in\{A^w,A^l\}$, let $A'_\tau$ denote the resulting noisy action.
We define the mean squared flow matching error as
\begin{equation}
    \bar{\ell}_\phi(A\mid x_t;\tau,\epsilon)
    =
    \frac{1}{H_{\mathrm{a}}D_a}
    \left\|
        \hat{v}_\phi(A'_\tau,x_t,\tau)
        -
        (\epsilon-A)
    \right\|_2^2,
    \label{eq:demodpo_flow_error}
\end{equation}
where $D_a$ is the action dimension.
The corresponding tractable surrogate for the log policy ratio is
\begin{equation}
    s_\theta(A\mid x_t;\tau,\epsilon)
    =
    -\frac{1}{2}
    \left[
    \bar{\ell}_\theta(A\mid x_t;\tau,\epsilon)
    -
    \bar{\ell}_{\mathrm{ref}}(A\mid x_t;\tau,\epsilon)
    \right].
    \label{eq:demodpo_ratio_score}
\end{equation}

Substituting the surrogate $s_\theta$ for $\rho_\theta$ in Equation~\ref{eq:standard_dpo} gives
\begin{equation}
    \mathcal{L}_{\mathrm{DemoDPO}}
    =
    -
    \mathbb{E}
    \left[
    \log\sigma
    \left(
    \beta
    \left[
    s_\theta(A^w\mid x_t;\tau,\epsilon)
    -
    s_\theta(A^l\mid x_t;\tau,\epsilon)
    \right]
    \right)
    \,\middle|\,
    M_{\mathrm{pair}}=1
    \right]
    .
    \label{eq:demodpo_objective}
\end{equation}
The expectation is over candidate groups sampled from the frozen reference policy and the shared flow noise, conditioned on $M_{\mathrm{pair}}=1$.
The loss encourages the trainable policy to assign a higher relative score to the candidate with lower demonstration MSE than to the candidate with higher demonstration MSE.
Appendix~\ref{app:demodpo_derivation} derives this objective from the Diffusion-DPO denoising surrogate and the Flow-DPO velocity formulation, and Algorithm~\ref{alg:demodpo} summarizes one update step.

\section{Experiments}
\label{sec:experiments}

\subsection{Experimental Setup}
\label{ssec:exp_setup}
\textbf{Simulation and real-world settings.}
We evaluate LeWAM on both simulated and real-world manipulation tasks.
For simulation, we use RoboTwin 2.0, a bimanual robotic manipulation benchmark comprising 50 tasks, following the multitask training setup of recent RoboTwin evaluations~\citep{gigaworld-policy, fast-wam}.
The training data contain 2,500 clean and 25,000 heavily randomized demonstrations, corresponding to 50 and 500 demonstrations per task, respectively.
We train all policies in our controlled comparisons for 10 epochs using AdamW with a learning rate of $10^{-4}$ and a global batch size of 1024 on 16 NVIDIA H800 GPUs.
Following Fast-WAM, we use an action chunk of 32 steps in simulation.
For future embedding prediction, we uniformly sample one future offset from steps 1 to 32 for each training sample and set $\alpha=0.1$.
For DemoDPO, we sample $G=4$ candidates from the frozen reference policy using 10 explicit Euler steps and retain pairs satisfying $d_{i^l}-d_{i^w}\geq10^{-3}$.
We train DemoDPO for 1,000 optimization steps with a learning rate of $10^{-5}$, $\beta=5000$, and an EMA decay of $0.9995$.
In the real world, we evaluate three bimanual tasks on an AgileX dual Piper robot: stacking three blocks, folding a towel, and arranging three flowers in a vase.

\textbf{Evaluation protocol.}
Simulation performance is measured by task success rate.
For the real-world evaluation, each method is evaluated for 50 rollouts per task, and we report task progress measured by milestone completion.
The complete real-world experiment details are provided in Appendix~\ref{app:real_world_protocol}.
Table~\ref{tab:inference_efficiency} in Appendix~\ref{app:inference_efficiency} reports inference latency and peak memory.

\subsection{RoboTwin 2.0 Results and Ablations}
\label{ssec:frozen_encoder_comparison}

We organize the RoboTwin 2.0 evaluation around three empirical questions: representation choice, latent world modeling, and preference refinement.
\begin{itemize}
    \item \textit{RQ1}: How do frozen visual representations affect action generation?
    \item \textit{RQ2}: How should future embeddings be predicted in JEPA space?
    \item \textit{RQ3}: What makes offline preference refinement effective for LeWAM?
\end{itemize}

\textbf{\textit{RQ1}: Frozen visual representations.}
Proposition~\ref{prop:gap} and the decomposition in Appendix~\ref{app:usable_action_uncertainty} attribute action uncertainty to the information a representation omits and the information a fixed predictor cannot extract, so we compare frozen vision encoders under the same action predictor and training recipe.
We use a discrete task ID as $c_t$ across all controlled comparisons to isolate visual representation effects from differences in text representations.
We first compare the Wan2.2 VAE, SigLIP, MAE-Large, DINOv3-Large, and V-JEPA2-Large, as shown in Figure~\ref{fig:main_results}\captiona{} and Table~\ref{tab:frozen_encoder_comparison}.
The Wan2.2 VAE performs worst at 10.81\%, while V-JEPA2-Large performs best among the three Large encoders.
Notably, I-JEPA-Huge outperforms V-JEPA2-Huge despite being pretrained only on ImageNet-22K~\citep{ImageNet}.

To further understand why I-JEPA supports action generation effectively, we visualize the attention from action tokens to vision tokens during inference.
Figure~\ref{fig:policy_attention} shows that the attention concentrates on the target objects and target regions and shifts between relevant regions as the manipulation progresses. 
Prior work links segmentation ability to manipulation generalization~\citep{burns2023robust} and shows benefits of geometric information for robot learning~\citep{SpatialBoost,wang2024visual}, so we examine how readily each frozen encoder exposes geometry and object information.
We train lightweight depth and segmentation probes on the frozen encoder outputs with teacher targets from DA3MONO-LARGE~\citep{DA3} and SAM3~\citep{SAM3}, following the protocol in Appendix~\ref{app:representation_probing}.
I-JEPA-Huge performs best among the evaluated individual encoders on both probes while the Wan2.2 VAE performs worst, indicating that I-JEPA makes the geometry and object information emphasized during action generation readily accessible.

\begin{figure}[htpb]
    \centering
    \includegraphics[width=0.95\linewidth]{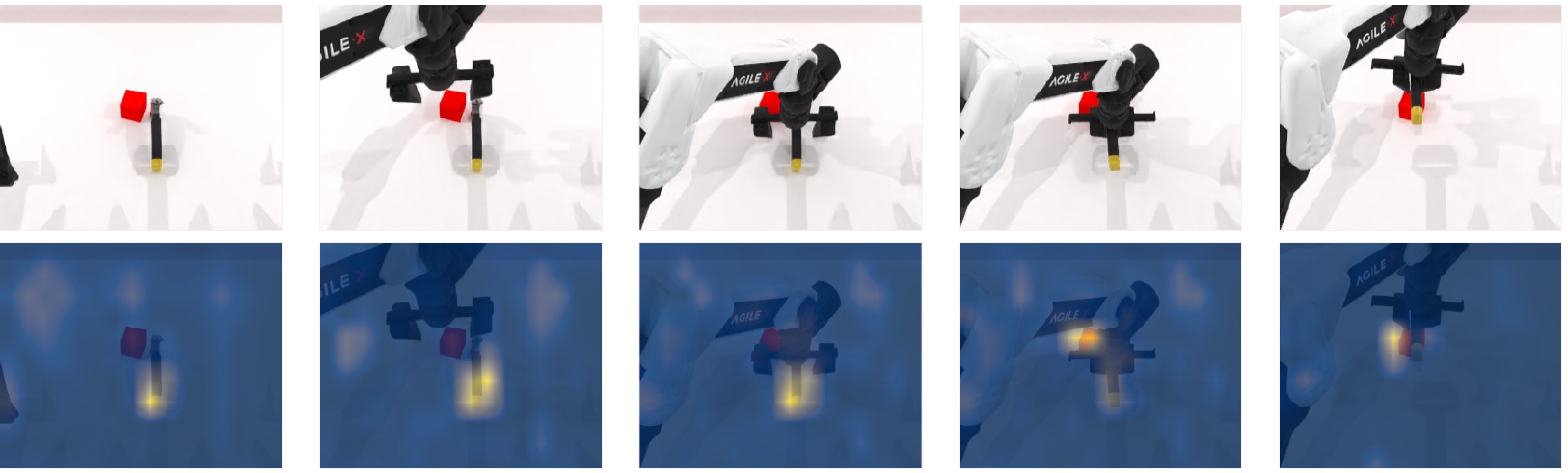}
    \vspace{-0.5em}
    \caption{Attention from action tokens to vision tokens during a representative hammering rollout.}
    \label{fig:policy_attention}
\end{figure}

\begin{table}[htpb]
  \centering
  \caption{Frozen encoder comparison on RoboTwin 2.0 with representation probing.
  Bold, underline, and italics mark the best, second, and third results, respectively.}
  \label{tab:frozen_encoder_comparison}
  \resizebox{0.95\linewidth}{!}{%
  \begin{tabular}{lccccccc}
    \hline
    & \multicolumn{3}{c}{Success rate} & \multicolumn{2}{c}{Depth} & \multicolumn{2}{c}{Segmentation} \\
    Encoder & Clean & Rand. & Avg. & RMSE $\downarrow$ & Corr. $\uparrow$ & AP@50 $\uparrow$ & AP@75 $\uparrow$ \\
    \hline
    Wan2.2-VAE & 12.52 & 9.10 & 10.81 & 0.708 & 0.600 & 0.206 & 0.025 \\
    SigLIP & 81.60 & 76.80 & 79.20 & 0.625 & 0.711 & 0.378 & 0.044 \\
    MAE-Large & 64.30 & 63.80 & 64.05 & 0.556 & 0.774 & 0.535 & 0.147 \\
    DINOv3-Large & 69.58 & 65.50 & 67.54 & 0.572 & 0.765 & 0.515 & 0.146 \\
    V-JEPA2-Large & 73.40 & 70.94 & 72.17 & 0.477 & 0.832 & \textit{0.571} & \textit{0.247} \\
    V-JEPA2-Huge & \textit{83.76} & \textit{79.64} & \textit{81.70} & \textit{0.446} & \textit{0.852} & 0.532 & 0.238 \\
    I-JEPA-Huge & \underline{88.42} & \underline{83.92} & \underline{86.17} & \underline{0.445} & \underline{0.853} & \underline{0.613} & \underline{0.323} \\
    + AdaFuse & \textbf{88.76} & \textbf{85.00} & \textbf{86.88} & \textbf{0.440} & \textbf{0.854} & \textbf{0.616} & \textbf{0.345} \\
    \hline
  \end{tabular}
  }
  \vspace{-0.4em}
\end{table}

\begin{table}[htpb]
  \centering
    \caption{RoboTwin baselines and LeWAM variants. We report average success rates over 50 tasks in clean and randomized settings, with 100 evaluation cases per task in each setting. 
    Embodied PT. indicates whether additional robot data are used for pretraining before benchmark training.}
  \label{tab:robotwin_main_results}
  \begin{tabular}{llccrrr}
    \hline
    \multicolumn{1}{l}{\raisebox{0.55em}{Model}} &
    \multicolumn{1}{l}{\raisebox{0.55em}{Encoder}} &
    \multicolumn{1}{l}{\shortstack[l]{Trainable\\Params.}} &
    \multicolumn{1}{l}{\shortstack[l]{Embodied\\PT.}} &
    \multicolumn{1}{l}{\raisebox{0.55em}{Clean}} &
    \multicolumn{1}{l}{\raisebox{0.55em}{Rand.}} &
    \multicolumn{1}{l}{\raisebox{0.55em}{Avg.}} \\

    \hline
    \multicolumn{7}{l}{\textit{VLM-based VLAs}} \\
    $\pi_{0}$ & SigLIP & 3B & \cmark & 65.92 & 58.40 & 62.16 \\
    $\pi_{0.5}$ & SigLIP & 3B & \cmark & 82.74 & 76.76 & 79.75 \\
    \hline
    \multicolumn{7}{l}{\textit{VDM-based WAMs}} \\
    Motus w/o Pretrain & Wan2.2-VAE & 6B & \xmark & 77.56 & 77.00 & 77.28 \\
    LingBot-VA w/o Pretrain & Wan2.2-VAE & 5B & \xmark & 80.60 & -- & 80.60 \\
    GigaWorld-Policy & Wan2.2-VAE & 5B & \cmark & 86.36 & 85.04 & 85.70 \\
    Motus & Wan2.2-VAE & 6B & \cmark & 88.66 & 87.02 & 87.84 \\
    Fast-WAM & Wan2.2-VAE & 6B & \xmark & 91.88 & 91.78 & 91.83 \\
    LingBot-VA & Wan2.2-VAE & 5B & \cmark & 92.90 & 91.50 & 92.20 \\
    \hline
    \multicolumn{7}{l}{\textit{Latent-based WAMs}} \\
    LaWAM & DINOv3-B/16 & 1.7B & \cmark & 92.64 & 89.80 & 91.22 \\
    \multicolumn{7}{c}{\textit{LeWAM Variants}} \\
    Action Only & I-JEPA-H/14 & 0.4B & \xmark & 88.42 & 83.92 & 86.17 \\
    + WM (FM) & I-JEPA-H/14 & 0.4B & \xmark & 88.18 & 84.24 & 86.21 \\
    + WM & I-JEPA-H/14 & 0.4B & \xmark & 88.86 & 85.18 & 87.02 \\
    + WM + AdaFuse & I-JEPA-H/14 & 0.4B & \xmark & 91.32 & 90.06 & 90.69 \\
    \rowcolor{blue!8} LeWAM & I-JEPA-H/14 & 0.4B & \xmark & \textbf{93.14} & \textbf{91.42} & \textbf{92.28} \\
    \hline
  \end{tabular}
\end{table}

\textbf{\textit{RQ2}: Future embedding prediction.}
We keep LeWAM's I-JEPA encoder frozen and compare direct prediction of future embeddings from learnable future queries, as described in Section~\ref{ssec:lewam_design}, with a flow matching alternative. The latter adds Gaussian noise to future embeddings and predicts their transport directions conditioned on the corresponding clean action prefixes.

Table~\ref{tab:robotwin_main_results} shows that direct prediction achieves 87.02\%, outperforming flow matching at 86.21\% and the Action Only baseline at 86.17\%.
Prior work shows that JEPA representations encode rich semantic and spatial structure~\citep{ijepa, vjepa21}.
In our setting, the frozen encoder maps each observed future frame to a fixed, high-dimensional target embedding, whereas flow matching introduces an additional noise-conditioned transport problem over this space.
We hypothesize that direct regression makes more direct use of this structured supervision, explaining its stronger performance.
We then evaluate AdaFuse, which adaptively integrates representations from multiple encoder layers.
Without future embedding prediction, AdaFuse improves success from 86.17\% to 86.88\%, as shown in Table~\ref{tab:frozen_encoder_comparison}.
Adding future embedding prediction further raises success to 90.69\% in Table~\ref{tab:robotwin_main_results}, with a larger improvement than without AdaFuse.
These results suggest that AdaFuse and future embedding prediction are complementary.
Table~\ref{tab:frozen_encoder_comparison} further shows that AdaFuse improves depth and segmentation probing over the final I-JEPA layer, suggesting that the learned fusion makes complementary information across encoder layers more accessible.

\textbf{\textit{RQ3}: DemoDPO refinement.}
DemoDPO improves the average success rate from 90.69\% to 92.28\% across all 50 RoboTwin tasks without additional environment interaction, as shown in Table~\ref{tab:robotwin_main_results}.
To examine what makes this refinement effective, Figure~\ref{fig:demodpo_ablations} compares pair filtering thresholds, DPO coefficients, and candidate group sizes on a fixed subset of eight challenging tasks.
The evaluation protocol is provided in Appendix~\ref{app:demodpo_ablation_details}.
Pair filtering retains candidates with a sufficient difference in demonstration similarity.
A moderate threshold gives the strongest improvement, whereas using all pairs or imposing an overly strict threshold reduces performance.
This suggests that effective refinement requires both distinguishable preferences and sufficient training pairs.
The DPO coefficient also affects refinement: performance improves as $\beta$ increases to 5000, then declines at 6000.
Increasing the candidate group size from 2 to 4 improves success, while increasing it further to 8 provides little additional benefit.
We therefore use $\varepsilon_{\mathrm{pair}}=10^{-3}$, $\beta=5000$, and $G=4$.

To distinguish the benefit of DemoDPO from that of additional training, we apply 1,000 additional supervised updates to the checkpoint used to initialize DemoDPO.
Success decreases slightly from 73.31\% to 73.25\%, whereas DemoDPO reaches 76.63\% after the same number of updates, indicating that its gains are not explained by additional training alone.
LeWAM therefore matches the strongest baselines in Table~\ref{tab:robotwin_main_results} with 0.4B trainable parameters and no video diffusion backbone, at the lowest latency and peak memory among the evaluated models in Table~\ref{tab:inference_efficiency}.

\begin{figure}[htpb]
    \centering
    \includegraphics[width=\linewidth]{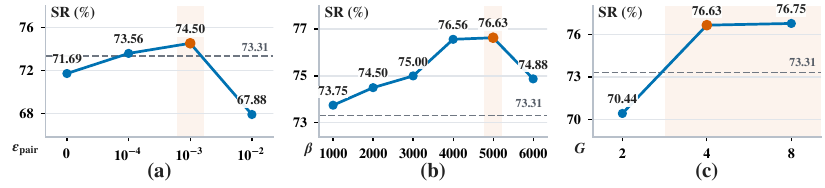}
    \vspace{-1.5em}
    \caption{DemoDPO analyses and ablations for RQ3.
    \captiona~Pair filtering threshold.
    \captionb~DPO coefficient.
    \captionc~Candidate group size.
    The dashed line marks the 73.31\% success rate before DemoDPO.}
    \label{fig:demodpo_ablations}
\end{figure}

\textbf{Real-world experiments.}
Figure~\ref{fig:main_results}\captionb{} summarizes the real-world results, with further results and implementation details provided in Appendix~\ref{app:real_world_protocol}.
\section{Conclusion}
\label{sec:conclusion}

Our theoretical analysis motivates examining both the information retained by visual representations and its accessibility for action generation.
Attention analysis shows that the policy focuses on target objects and regions during manipulation, while controlled encoder comparisons and representation probing suggest that I-JEPA makes the relevant geometric and object information more accessible than raw VAE representations.
These findings motivate LeWAM, which couples action generation with future embedding prediction in I-JEPA space rather than relying on a pretrained video diffusion backbone.
AdaFuse strengthens this representation by integrating complementary information across encoder layers, while DemoDPO further refines the policy using preferences derived from demonstrations.
Together, these components enable LeWAM to achieve an average success rate of 92.28\% on RoboTwin 2.0 and strong performance on real-world manipulation tasks, with only 0.4B trainable parameters and no embodied pretraining.
\subsection*{AI use statement}



In this work, we used generative AI tools to assist with codebase development and manuscript polishing. Generative AI was not used to produce conclusions unsupported by objective experimental evidence. All AI-assisted content, including but not limited to code and manuscript text, was manually reviewed and verified by the authors. We take full responsibility for the final content of this work.




\subsection*{Reproducibility statement}



To support reproducibility, we provide the training and inference code in the supplementary material. Our simulation experiments use the official RoboTwin 2.0 training dataset. The real-world experiments use an in-house dataset collected through teleoperation, which we will publicly release. Detailed model configurations, training procedures, and evaluation protocols are provided in the paper and appendix.



\bibliography{iclr2027_conference}
\bibliographystyle{iclr2027_conference}

\newpage
\appendix
\section*{Appendix}

\section{Theoretical Analysis}
\label{app:mi_wam}

This section proves Propositions~\ref{prop:gap} and~\ref{prop:vae} and analyzes the effective dimension and action information of VAE posterior means.

\subsection{Action-Relevant Information Bounds}

We use the same notation as the main text: $X$ is the current visual observation, $S$ denotes the action-relevant explanatory factors, $A$ is the expert action chunk, $c$ is task context, and $Z=E(X)$ is the visual representation consumed by the action predictor.

\begin{proof}[Proof of Proposition~\ref{prop:gap}]
\renewcommand{\qedsymbol}{}
Let $Z=E(X)$ and suppose $H(A\mid Z,c)$ and $H(A\mid X,c)$ are finite.
Because $Z$ is determined by $X$, conditioning on $(X,Z,c)$ is equivalent to conditioning on $(X,c)$.
\begin{equation}
    I(A;X\mid Z,c)
    =
    H(A\mid Z,c)-H(A\mid X,Z,c)
    =
    H(A\mid Z,c)-H(A\mid X,c)
    =
    \Delta_E(A).
\end{equation}
Since $S=S(X)$, the mutual information chain rule gives
\begin{equation}
    I(A;X\mid Z,c)
    =
    I(A;S,X\mid Z,c)
    =
    I(A;S\mid Z,c)+I(A;X\mid S,Z,c).
\end{equation}
Assumption~\ref{as:sufficiency} states that $A$ is conditionally independent of $X$ given $(S,c)$.
Since $Z$ is a function of $X$, this also gives $I(A;X\mid S,Z,c)=0$.
Combining the two identities and using nonnegativity of conditional mutual information yields
\begin{equation}
    \Delta_E(A)
    =
    I(A;S\mid Z,c)
    \geq 0.
\end{equation}
\end{proof}
This identity motivates evaluating a representation by the action-relevant factor information it preserves rather than by its image reconstruction quality.

\subsection{Spectral Compression of VAE Posterior Means}
\label{app:vae_spectral_compression}

We analyze how VAE training selects the effective dimension of the posterior mean, then determine when the removed directions contain action information.
The spectral analysis specializes the linear VAE result of \citet{wang2022posterior}. The action uncertainty result below follows by conditioning on the retained directions.
This analysis concerns a linear Gaussian model at a global optimum, with a fixed decoder variance.
The linear Gaussian model is an analyzable instance, not a model of the pretrained Wan2.2 VAE, whose encoder is nonlinear and whose pretraining distribution differs from the demonstration distribution.
We use it to establish what reconstruction based selection does and does not guarantee about action information, and we treat the pretrained encoder empirically in Section~\ref{ssec:frozen_encoder_comparison}.

Fix a task context $c$ and let the centered observation vector be $X\in\mathbb{R}^{d_x}$ with
\begin{equation}
    X\mid c\sim\mathcal{N}(0,\Sigma_X),
    \qquad
    \Sigma_X=\sum_{i=1}^{d_x}\lambda_i u_i u_i^\top,
    \qquad
    \lambda_1>\cdots>\lambda_{d_x}>0.
    \label{eq:vae_spectral_data}
\end{equation}
The $u_i$ are orthonormal principal directions, and $\lambda_i$ measures observation variance along $u_i$.
Consider a VAE with $k\leq d_x$ latent coordinates, standard normal prior, and
\begin{equation}
    q(\tilde z\mid x)=\mathcal{N}(Mx,\operatorname{diag}(s_1^2,\ldots,s_k^2)),
    \qquad
    p(x\mid\tilde z)=\mathcal{N}(W\tilde z,\sigma_{\mathrm{dec}}^2\mathbf{I}),
    \label{eq:vae_spectral_model}
\end{equation}
where $M\in\mathbb{R}^{k\times d_x}$, $W\in\mathbb{R}^{d_x\times k}$, and $s_i^2>0$ are learned, while $\sigma_{\mathrm{dec}}^2>0$ is fixed.
The encoder variances are independent of $x$.
Let $\widetilde Z$ denote the sampled latent variable and $\tilde z$ its realization.
The training objective is, up to a constant,
\begin{equation}
    \mathcal{L}_{\mathrm{vae}}
    =
    \frac{1}{2\sigma_{\mathrm{dec}}^2}
    \mathbb{E}_{X\mid c}
    \mathbb{E}_{\widetilde Z\mid X}
    \left[\|X-W\widetilde Z\|_2^2\right]
    +
    \beta_{\mathrm{vae}}\,
    \mathbb{E}_{X\mid c}
    D_{\mathrm{KL}}\bigl(q(\tilde z\mid X)\|\mathcal{N}(0,\mathbf{I})\bigr).
    \label{eq:vae_spectral_objective}
\end{equation}
Here $\beta_{\mathrm{vae}}>0$ is the VAE regularization weight.
Although training samples $\widetilde Z$, the representation supplied to the action predictor is the deterministic posterior mean
$Z_{\mathrm{vae}}=\mathbb{E}[\widetilde Z\mid X]=MX$.
Let $\lambda_{\mathrm{cut}}=\beta_{\mathrm{vae}}\sigma_{\mathrm{dec}}^2$.
At a global optimum expressed in principal coordinates, the posterior mean has coordinates
\begin{equation}
    (Z_{\mathrm{vae}})_i
    =
    \frac{\sqrt{(\lambda_i-\lambda_{\mathrm{cut}})_+}}{\lambda_i}
    u_i^\top X,
    \qquad i=1,\ldots,k,
    \label{eq:vae_spectral_mean}
\end{equation}
where $(a)_+=\max(a,0)$.
Its effective dimension, defined as the rank of its covariance, is
\begin{equation}
    r_{\mathrm{eff}}
    :=
    \operatorname{rank}\operatorname{Cov}(Z_{\mathrm{vae}}\mid c)
    =
    \min\bigl(k,\#\{i:\lambda_i>\lambda_{\mathrm{cut}}\}\bigr).
    \label{eq:vae_effective_dimension}
\end{equation}
Let $J$ be the index set of retained directions, with $J=\{i\leq k:\lambda_i>\lambda_{\mathrm{cut}}\}$, and define
$\Sigma_{\mathrm{omit}}=\sum_{i\notin J}\lambda_i u_i u_i^\top$.
If the expert action chunk, viewed as a vector, satisfies
\begin{equation}
    A=B_cX+\xi,
    \qquad
    \xi\sim\mathcal{N}(0,\Sigma_\xi),
    \qquad
    \xi\perp X\mid c,
    \qquad
    \Sigma_\xi\succ0,
    \label{eq:vae_linear_action}
\end{equation}
then the additional action uncertainty is
\begin{equation}
    \Delta_E(A)
    =
    I(A;X\mid Z_{\mathrm{vae}},c)
    =
    \frac12\log\det\left(
        \mathbf{I}
        +
        \Sigma_\xi^{-1/2}B_c\Sigma_{\mathrm{omit}}B_c^\top\Sigma_\xi^{-1/2}
    \right).
    \label{eq:vae_spectral_action_gap}
\end{equation}
It is strictly positive exactly when $B_cu_i\neq0$ for at least one omitted direction.

\begin{proof}[Proof of Proposition~\ref{prop:vae}]
\renewcommand{\qedsymbol}{}
\textbf{1. Reduce training to principal directions.}
The spectral solution for the linear VAE aligns the decoder with principal directions of $\Sigma_X$ \citep[Theorem 2]{wang2022posterior}.
For completeness, we derive the resulting threshold and its effect on the posterior mean.
Write $x_i=u_i^\top X$, and denote the scalar decoder loading, encoder mean coefficient, and posterior variance along this direction by $w_i$, $m_i$, and $s_i^2$.
The contribution of this direction to Equation~\ref{eq:vae_spectral_objective} is
\begin{equation}
    \ell_i
    =
    \frac{\lambda_i(1-w_im_i)^2+w_i^2s_i^2}{2\sigma_{\mathrm{dec}}^2}
    +
    \frac{\beta_{\mathrm{vae}}}{2}
    \left(\lambda_i m_i^2+s_i^2-\log s_i^2-1\right).
    \label{eq:vae_scalar_objective}
\end{equation}
The first term rewards reconstruction, while the second penalizes the encoder mean and deviation of the posterior variance from the prior variance.

\textbf{2. Identify which directions survive regularization.}
Fix $w_i$. Minimizing over $m_i$ and $s_i^2$ gives
\begin{equation}
    m_i^*=\frac{w_i}{w_i^2+\lambda_{\mathrm{cut}}},
    \qquad
    (s_i^2)^*=\frac{\lambda_{\mathrm{cut}}}{w_i^2+\lambda_{\mathrm{cut}}}.
    \label{eq:vae_optimal_encoder}
\end{equation}
Substituting these values and writing $t_i=w_i^2\geq0$ reduces the objective to
\begin{equation}
    \ell_i^*(t_i)
    =
    \frac{\beta_{\mathrm{vae}}}{2}
    \left[
        \frac{\lambda_i}{t_i+\lambda_{\mathrm{cut}}}
        +
        \log\left(1+\frac{t_i}{\lambda_{\mathrm{cut}}}\right)
    \right],
    \qquad
    \frac{d\ell_i^*}{dt_i}
    =
    \frac{\beta_{\mathrm{vae}}}{2}
    \frac{t_i+\lambda_{\mathrm{cut}}-\lambda_i}
    {(t_i+\lambda_{\mathrm{cut}})^2}.
    \label{eq:vae_spectral_threshold}
\end{equation}
Thus $(w_i^2)^*=(\lambda_i-\lambda_{\mathrm{cut}})_+$.
Choosing nonnegative decoder loadings gives Equation~\ref{eq:vae_spectral_mean}. Sign changes or permutations of latent coordinates do not affect the result.
When $\lambda_i\leq\lambda_{\mathrm{cut}}$, both $w_i^*$ and $m_i^*$ vanish, and $(s_i^2)^*=1$: this coordinate matches the prior and its mean carries no observation information.
When $\lambda_i>\lambda_{\mathrm{cut}}$, its mean coefficient is nonzero.
The reduction in the optimized loss relative to omitting a direction increases with $\lambda_i$ above the threshold, so the $k$ available coordinates retain the largest eligible eigenvalues.
Their number is given by Equation~\ref{eq:vae_effective_dimension}.

\textbf{3. Determine the information missing from the posterior mean.}
For every $i\in J$, the nonzero coefficient in Equation~\ref{eq:vae_spectral_mean} makes $u_i^\top X$ recoverable from $Z_{\mathrm{vae}}$.
For $i\notin J$, the corresponding coordinate is absent.
The Gaussian principal coordinates are independent, so
\begin{equation}
    \mathbb{E}[X\mid Z_{\mathrm{vae}},c]
    =
    \sum_{i\in J}u_i u_i^\top X,
    \qquad
    \operatorname{Cov}(X\mid Z_{\mathrm{vae}},c)
    =
    \Sigma_{\mathrm{omit}}.
    \label{eq:vae_missing_covariance}
\end{equation}
This identifies the lost directions for the deterministic mean, rather than a noisy latent sample.

\textbf{4. Translate missing observation directions into action uncertainty.}
Applying Equation~\ref{eq:vae_linear_action} gives
\begin{equation}
    \operatorname{Cov}(A\mid X,c)=\Sigma_\xi,
    \qquad
    \operatorname{Cov}(A\mid Z_{\mathrm{vae}},c)
    =
    \Sigma_\xi+B_c\Sigma_{\mathrm{omit}}B_c^\top.
    \label{eq:vae_action_covariance}
\end{equation}
Both conditional action distributions are Gaussian with positive definite covariance.
Their differential entropies are finite, and subtracting them yields
\begin{equation}
    \Delta_E(A)
    =
    H(A\mid Z_{\mathrm{vae}},c)-H(A\mid X,c)
    =
    \frac12\log
    \frac{\det(\Sigma_\xi+B_c\Sigma_{\mathrm{omit}}B_c^\top)}
    {\det\Sigma_\xi},
    \label{eq:vae_action_entropy_gap}
\end{equation}
which equals the expression in Equation~\ref{eq:vae_spectral_action_gap}.
Here $H$ denotes differential entropy. No nonnegativity assumption on differential entropy is used.
Finally,
$B_c\Sigma_{\mathrm{omit}}B_c^\top
=\sum_{i\notin J}\lambda_i(B_cu_i)(B_cu_i)^\top$
is positive semidefinite and is nonzero exactly when an omitted direction affects the action.
The log determinant gap is therefore strictly positive precisely under the stated condition.
\end{proof}

\begin{corollary}[Reconstruction ordering does not order action relevance]\label{cor:misalign}
In the setting of Proposition~\ref{prop:vae} with $k<d_x$, let $B_c=b\,u_{d_x}^{\top}$, so the action depends only on the least energetic observation direction.
Then $\Delta_E(A)=I(A;X\mid c)$, that is, $Z_{\mathrm{vae}}$ retains no action information, whereas the single direction $u_{d_x}^{\top}X$ attains $\Delta_E(A)=0$ while having the largest reconstruction error among principal directions.
There exist sequences of observation and action noise covariances for which the minimum reconstruction error from $Z_{\mathrm{vae}}$ tends to zero while $\Delta_E(A)$ diverges.
\end{corollary}

\begin{proof}[Proof of Corollary~\ref{cor:misalign}]
\renewcommand{\qedsymbol}{}
Since $k<d_x$, the index $d_x$ is omitted, so $\lambda_{d_x}u_{d_x}u_{d_x}^{\top}\preceq\Sigma_{\mathrm{omit}}$ and
$B_c\Sigma_{\mathrm{omit}}B_c^{\top}=\lambda_{d_x}bb^{\top}=B_c\Sigma_XB_c^{\top}$,
using $B_c=b\,u_{d_x}^{\top}$ and $u_{d_x}^{\top}\Sigma_Xu_{d_x}=\lambda_{d_x}$.
Substituting into Equation~\ref{eq:vae_spectral_action_gap} and applying $\det(\mathbf{I}+vv^{\top})=1+\|v\|_2^2$ gives
\begin{equation}
    \Delta_E(A)
    =
    \frac12\log\bigl(1+\lambda_{d_x}b^{\top}\Sigma_\xi^{-1}b\bigr)
    =
    \frac12\log\frac{\det(\Sigma_\xi+B_c\Sigma_XB_c^{\top})}{\det\Sigma_\xi}
    =
    I(A;X\mid c),
\end{equation}
so the posterior mean retains none of the action information in $X$.
For the representation $u_{d_x}^{\top}X$, the action is conditionally independent of $X$ given $(u_{d_x}^{\top}X,c)$, so $\Delta_E(A)=0$, while its reconstruction error $\operatorname{tr}\Sigma_X-\lambda_{d_x}$ is the largest among projections onto a single principal direction.
Fix $\lambda_1>\cdots>\lambda_k>\lambda_{\mathrm{cut}}$ and take $0<\epsilon<\lambda_{\mathrm{cut}}$.
Choose distinct $\lambda_i\in(\epsilon/2,\epsilon]$ for all $i>k$, together with $\Sigma_\xi=\epsilon^{2}\mathbf{I}$ and $\|b\|_2=1$.
Then $J=\{1,\ldots,k\}$.
Then $\operatorname{tr}\Sigma_{\mathrm{omit}}\leq(d_x-k)\epsilon\to0$ while $\lambda_{d_x}b^{\top}\Sigma_\xi^{-1}b\geq1/(2\epsilon)\to\infty$.
\end{proof}

The result separates two sources of dimensional reduction: the architectural width $k$ and the threshold $\beta_{\mathrm{vae}}\sigma_{\mathrm{dec}}^2$ induced by training.
Within this model, directions are retained according to their contribution to visual reconstruction, while their effect on action uncertainty depends additionally on $B_c$.
The effective dimension can be smaller than the latent width, and losing a direction increases action uncertainty only when that direction matters for the task.
Corollary~\ref{cor:misalign} and Proposition~\ref{prop:loss_bound} bound the two ends of this behavior: reconstruction based selection can retain no action information when the action depends on a low variance direction, and it can lose only a controlled amount when reconstruction is accurate and the action sensitivity is bounded.

\subsection{Usable Action Uncertainty}
\label{app:usable_action_uncertainty}

Proposition~\ref{prop:gap} measures information content and is indifferent to whether a downstream model can extract it.
We record here the decomposition used in Section~\ref{ssec:action_sufficient_latents} and the bound that separates its two terms.
Let $\mathcal{V}$ be a family of conditional densities for $A$ given $(Z,c)$ and let
$H_{\mathcal{V}}(A\mid Z,c)=\inf_{f\in\mathcal{V}}\mathbb{E}[-\log f(A\mid Z,c)]$
be the usable conditional entropy of \citet{xu2020usable}.
For any $f$,
$\mathbb{E}[-\log f(A\mid Z,c)]-H(A\mid Z,c)
=\mathbb{E}_{Z,c}D_{\mathrm{KL}}\bigl(p(\cdot\mid Z,c)\,\|\,f(\cdot\mid Z,c)\bigr)\geq0$,
so $H_{\mathcal{V}}(A\mid Z,c)\geq H(A\mid Z,c)$, with equality when $\mathcal{V}$ contains the true conditional.
Adding and subtracting $H(A\mid Z,c)$ yields
\begin{equation}
    \underbrace{H_{\mathcal{V}}(A\mid Z,c)-H(A\mid X,c)}_{\text{optimal excess log loss}}
    =
    \underbrace{\Delta_E(A)}_{\text{information loss}}
    +
    \underbrace{H_{\mathcal{V}}(A\mid Z,c)-H(A\mid Z,c)}_{\text{usability deficit}},
    \label{eq:usable_decomposition}
\end{equation}
whose first term is nonnegative by Proposition~\ref{prop:gap} and whose second term is nonnegative by the inequality above.
The first term depends only on the representation, while the second depends jointly on the representation and the downstream model.
Here $A$ denotes the vectorized action chunk of dimension $H_{\mathrm a}D_a$.
If $\mathcal{V}$ consists of predictors $\mathcal{N}(g(Z,c),\sigma^{2}\mathbf{I})$ with $g$ ranging over the realizable regressors and $\sigma^{2}>0$ free, and $0<\overline{M}(Z)<\infty$, then taking the infimum over $g$ and optimizing $\sigma^{2}$ gives $\sigma^{2}=\overline{M}(Z)/(H_{\mathrm a}D_a)$ and
\begin{equation}
    H_{\mathcal{V}}(A\mid Z,c)
    =
    \frac{H_{\mathrm a}D_a}{2}\log\!\left(\frac{2\pi e\,\overline{M}(Z)}{H_{\mathrm a}D_a}\right),
    \qquad
    \overline{M}(Z)=\inf_{g}\mathbb{E}\left[\|A-g(Z,c)\|_2^2\right].
\end{equation}
For this Gaussian family, usable conditional entropy is monotone in the minimum achievable action MSE.
\begin{proposition}[Loss channel bound]\label{prop:loss_bound}
In the setting of Proposition~\ref{prop:vae} with $A=B_cX+\xi$ as in Equation~\ref{eq:vae_linear_action},
\begin{equation}
    \Delta_E(A)
    \leq
    \frac{H_{\mathrm a}D_a}{2}
    \log\!\left(1+\frac{\|B_c\|_2^2\,\mathrm{MSE}_{\mathrm{rec}}}{\lambda_{\min}(\Sigma_\xi)}\right),
    \quad
    \mathrm{MSE}_{\mathrm{rec}}
    =
    \operatorname{tr}\Sigma_{\mathrm{omit}}
    =
    \mathbb{E}\bigl[\|X-\mathbb{E}[X\mid Z_{\mathrm{vae}},c]\|_2^2\bigr].
\end{equation}
\end{proposition}

\begin{proof}
\renewcommand{\qedsymbol}{}
The identity for $\mathrm{MSE}_{\mathrm{rec}}$ follows from Equation~\ref{eq:vae_missing_covariance}.
Since $\Sigma_{\mathrm{omit}}\succeq0$, we have
$\Sigma_{\mathrm{omit}}\preceq\lambda_{\max}(\Sigma_{\mathrm{omit}})\mathbf{I}\preceq(\operatorname{tr}\Sigma_{\mathrm{omit}})\mathbf{I}$, hence
$B_c\Sigma_{\mathrm{omit}}B_c^{\top}\preceq\mathrm{MSE}_{\mathrm{rec}}\,B_cB_c^{\top}\preceq\mathrm{MSE}_{\mathrm{rec}}\|B_c\|_2^{2}\mathbf{I}$.
Conjugating by $\Sigma_\xi^{-1/2}$ and using $\Sigma_\xi^{-1}\preceq\lambda_{\min}(\Sigma_\xi)^{-1}\mathbf{I}$ gives
\begin{equation}
    \Sigma_\xi^{-1/2}B_c\Sigma_{\mathrm{omit}}B_c^{\top}\Sigma_\xi^{-1/2}
    \preceq
    \frac{\|B_c\|_2^{2}\,\mathrm{MSE}_{\mathrm{rec}}}{\lambda_{\min}(\Sigma_\xi)}\mathbf{I}.
\end{equation}
The log determinant is monotone in the positive semidefinite order, so substituting this bound into Equation~\ref{eq:vae_spectral_action_gap} and evaluating the determinant of the resulting multiple of the identity gives the claim.
\end{proof}

Because $\mathrm{MSE}_{\mathrm{rec}}$ is the minimum mean squared error of predicting $X$ from $Z_{\mathrm{vae}}$, it is bounded above by the error of any particular decoder, including one acting on the sampled latent $\widetilde Z$, since $X\to Z_{\mathrm{vae}}\to\widetilde Z$ forms a Markov chain.
Measured reconstruction fidelity therefore upper bounds the information loss term within this model.
Corollary~\ref{cor:misalign} operates in the regime where the action sensitivity $\|B_c\|_2^2/\lambda_{\min}(\Sigma_\xi)$ diverges and the bound becomes vacuous, so the two results are compatible.
The bound relates reconstruction error to action information loss through the action sensitivity and noise level in this linear Gaussian model.

\subsection{Limits of Downstream Recovery}
Let a downstream model construct a hidden state $R$ deterministically from $(Z,c)$. The data processing inequality gives
\begin{equation}
    I(S;R)
    \leq
    I(S;Z,c).
\end{equation}
Thus, a trainable transformer can reorganize retained information into a predictor state, but it cannot recreate action-relevant information discarded jointly by the representation and context.

\section{DemoDPO Derivation and Algorithm}
\label{app:demodpo_details}

We derive the DemoDPO objective by following Diffusion-DPO from action likelihoods to a surrogate based on denoising errors, then converting these errors to flow matching errors.
The derivation separates the path approximation, the bound used for sampling a single step, and the weighting choice used in training.

\subsection{Objective Derivation}
\label{app:demodpo_derivation}

\paragraph{1. Preference objective.}
Let $\mathcal{D}=\{(x_t,A^w,A^l)\}$ contain preference pairs under the training condition $x_t=(X_t,c_t)$.
Using the notation of Equation~\ref{eq:standard_dpo}, standard DPO optimizes
\begin{equation}
    \mathcal{L}_{\mathrm{DPO}}
    =
    -\mathbb{E}_{\mathcal{D}}
    \log\sigma\left(
        \beta\left[
            \rho_\theta(A^w\mid x_t)-\rho_\theta(A^l\mid x_t)
        \right]
    \right),
    \qquad
    \rho_\theta(A\mid x_t)
    =
    \log\frac{\pi_\theta(A\mid x_t)}{\pi_{\mathrm{ref}}(A\mid x_t)}.
    \label{eq:diffusion_dpo_endpoint}
\end{equation}
This objective favors the preferred action relative to the reference policy.
Its direct evaluation requires action likelihoods, which are not outputs of the velocity predictor.
We therefore follow the diffusion surrogate construction of \citet{diffusion_dpo}, before adapting it to flow matching.

\paragraph{2. From actions to generation paths.}
In the diffusion construction, let $A_0=A$ and let $A_{1:K}$ denote intermediate noisy actions.
Here $k$ indexes diffusion steps, not control steps.
For $\phi\in\{\theta,\mathrm{ref}\}$, the path distribution and its endpoint marginal are
\begin{equation}
    p_\phi(A_{0:K}\mid x_t)
    =
    p(A_K)\prod_{k=1}^{K}p_\phi(A_{k-1}\mid A_k,x_t),
    \qquad
    p_\phi(A\mid x_t)
    =
    \int p_\phi(A_{0:K}\mid x_t)\,dA_{1:K}.
    \label{eq:diffusion_dpo_reverse_process}
\end{equation}
The marginal integrates over all paths ending at $A$, whereas the path distribution factorizes into individual transitions.
Diffusion-DPO constructs a preference objective on paths, using the joint KL to upper bound the endpoint KL:
\begin{equation}
    D_{\mathrm{KL}}\left(
        p_\theta(A_0\mid x_t)\|p_{\mathrm{ref}}(A_0\mid x_t)
    \right)
    \leq
    D_{\mathrm{KL}}\left(
        p_\theta(A_{0:K}\mid x_t)\|p_{\mathrm{ref}}(A_{0:K}\mid x_t)
    \right).
    \label{eq:diffusion_dpo_path_kl}
\end{equation}
The corresponding score averages the path log ratio conditional on the endpoint:
\begin{equation}
    \rho_\theta^{\mathrm{path}}(A\mid x_t)
    =
    \mathbb{E}_{p_\theta(A_{1:K}\mid A,x_t)}
    \left[
        \log\frac{p_\theta(A_{0:K}\mid x_t)}
        {p_{\mathrm{ref}}(A_{0:K}\mid x_t)}
    \right].
    \label{eq:diffusion_dpo_path_score}
\end{equation}
Replacing $\rho_\theta$ by this score gives a surrogate preference objective on paths.
The bound in Equation~\ref{eq:diffusion_dpo_path_kl} applies to the endpoint KL,
not the original DPO loss.
The shared terminal prior cancels in the path ratio, leaving a sum of transition log ratios.

\paragraph{3. From paths to a sampled step.}
Sampling paths conditional on a given endpoint remains intractable.
Following Diffusion-DPO, we approximate this posterior by the known forward noising process $q(A_{1:K}\mid A)$.
Define the expected transition log ratio
\begin{equation}
    g_{\theta,k}(A,A_k\mid x_t)
    :=
    \mathbb{E}_{q(A_{k-1}\mid A_k,A)}
    \left[
        \log\frac{p_\theta(A_{k-1}\mid A_k,x_t)}
        {p_{\mathrm{ref}}(A_{k-1}\mid A_k,x_t)}
    \right].
    \label{eq:diffusion_dpo_step_score}
\end{equation}
At $k=1$, $A_0=A$ is fixed, so this is the log density ratio for the final denoising step.
The approximated path score becomes
\begin{equation}
    \rho_\theta^q(A\mid x_t)
    :=
    \mathbb{E}_{q(A_{1:K}\mid A)}
    \log\frac{p_\theta(A_{0:K}\mid x_t)}
    {p_{\mathrm{ref}}(A_{0:K}\mid x_t)}
    =
    K\,\mathbb{E}_{k,A_k\sim q(\cdot\mid A)}
    g_{\theta,k}(A,A_k\mid x_t),
    \label{eq:diffusion_dpo_forward_score}
\end{equation}
where $k$ is uniform on $\{1,\ldots,K\}$.
Thus, the path sum can be estimated by sampling a single step.
For a fixed pair, write $g_{\theta,k}^{w}$ and $g_{\theta,k}^{l}$ for the two transition scores at the same sampled step.
Since $-\log\sigma$ is convex, Jensen's inequality gives
\begin{equation}
    \mathcal{L}_{q}
    :=
    -\mathbb{E}_{\mathcal{D}}
    \log\sigma\left(
        \beta K\,\mathbb{E}_{k,q}
        [g_{\theta,k}^{w}-g_{\theta,k}^{l}]
    \right)
    \leq
    -\mathbb{E}_{\mathcal{D},k,q}
    \log\sigma\left(
        \beta K[g_{\theta,k}^{w}-g_{\theta,k}^{l}]
    \right)
    =:\mathcal{L}_{\mathrm{step}}.
    \label{eq:diffusion_dpo_step_bound}
\end{equation}
Here $q$ supplies the noisy marginals for both actions.
This bound makes training possible with one sampled noise level per pair.
It applies to the objective after the forward process approximation.

\paragraph{4. From Gaussian transitions to denoising errors.}
The remaining transition scores can be evaluated through noise prediction.
For $k\geq2$, inserting the forward posterior into the log ratio gives
\begin{align}
    g_{\theta,k}(A,A_k\mid x_t)
    &=
    D_{\mathrm{KL}}\left(
        q(A_{k-1}\mid A_k,A)\|p_{\mathrm{ref}}(A_{k-1}\mid A_k,x_t)
    \right)
    \nonumber\\
    &\quad-
    D_{\mathrm{KL}}\left(
        q(A_{k-1}\mid A_k,A)\|p_\theta(A_{k-1}\mid A_k,x_t)
    \right).
    \label{eq:diffusion_dpo_step_kl}
\end{align}
In the Gaussian diffusion construction, the forward posterior has mean $\mu_{q,k}$, and the reverse models have means $\mu_{\phi,k}$ and a shared fixed covariance $\nu_k^2\mathbf{I}$.
Terms involving the posterior covariance cancel between the two KLs.
Under the usual noise parameterization, $\mu_{\phi,k}-\mu_{q,k}=b_k(\epsilon-\hat\epsilon_\phi)$, where $b_k$ depends only on the noise schedule.
Consequently,
\begin{align}
    g_{\theta,k}(A,A_k\mid x_t)
    &=
    -\frac{1}{2\nu_k^2}
    \left[
        \|\mu_{\theta,k}-\mu_{q,k}\|_2^2
        -
        \|\mu_{\mathrm{ref},k}-\mu_{q,k}\|_2^2
    \right]
    \nonumber\\
    &=
    -\omega_k
    \left[
        \|\hat\epsilon_\theta-\epsilon\|_2^2
        -
        \|\hat\epsilon_{\mathrm{ref}}-\epsilon\|_2^2
    \right],
    \qquad
    \omega_k=\frac{b_k^2}{2\nu_k^2}.
    \label{eq:diffusion_dpo_noise_score}
\end{align}
The noise predictors are evaluated on the same $A_k$, condition $x_t$, and step $k$.
For a Gaussian endpoint decoder with shared fixed variance, the $k=1$ log density ratio likewise reduces to a difference of reconstruction errors, which can be expressed as noise errors with its corresponding weight.
Substituting these scores into Equation~\ref{eq:diffusion_dpo_step_bound} yields the denoising error surrogate: the preferred action should have a smaller error relative to the reference model than the rejected action.

\paragraph{5. From denoising to flow matching.}
To adapt this surrogate to our velocity predictor, we follow \citet{flowdpo} and use the linear interpolation from Equation~\ref{eq:action_interpolation}:
\begin{equation}
    A'_\tau=(1-\tau)A+\tau\epsilon,
    \qquad
    v_\tau=\epsilon-A,
    \qquad
    \epsilon\sim\mathcal{N}(0,\mathbf{I}).
    \label{eq:flow_dpo_linear_path}
\end{equation}
Since $\epsilon=A'_\tau+(1-\tau)v_\tau$, the velocity predictor induces a noise predictor through
\begin{equation}
    \hat\epsilon_\phi
    :=
    A'_\tau+(1-\tau)\hat v_\phi(A'_\tau,x_t,\tau),
    \qquad
    \phi\in\{\theta,\mathrm{ref}\}.
    \label{eq:flow_dpo_noise_predictor}
\end{equation}
Subtracting the true noise gives the exact error conversion
\begin{equation}
    \|\hat\epsilon_\phi-\epsilon\|_2^2
    =
    (1-\tau)^2
    \|\hat v_\phi(A'_\tau,x_t,\tau)-v_\tau\|_2^2.
    \label{eq:flow_dpo_error_conversion}
\end{equation}
Thus, the denoising surrogate becomes a preference loss on velocity errors with a weight that depends on $\tau$.
Following the constant weighting used by Flow-DPO, we replace the resulting weight by a constant.
This is a choice of training objective, separate from the exact error conversion above.

\paragraph{6. DemoDPO objective.}
We now use the normalization and score from Equations~\ref{eq:demodpo_flow_error} and~\ref{eq:demodpo_ratio_score}:
\begin{align}
    \bar\ell_\phi(A\mid x_t;\tau,\epsilon)
    &:=
    \frac{1}{H_{\mathrm a}D_a}
    \left\|
        \hat v_\phi(A'_\tau,x_t,\tau)-(\epsilon-A)
    \right\|_2^2,
    \label{eq:flow_dpo_single_time_error}\\
    s_\theta(A\mid x_t;\tau,\epsilon)
    &:=
    -\frac12
    \left[
        \bar\ell_\theta(A\mid x_t;\tau,\epsilon)
        -
        \bar\ell_{\mathrm{ref}}(A\mid x_t;\tau,\epsilon)
    \right].
    \label{eq:flow_dpo_sampled_ratio}
\end{align}
All constant scale factors are absorbed into $\beta$, with $1/2$ kept explicit to match the main text.
For each candidate pair, we sample $\tau\sim\operatorname{Unif}[0,1]$ and a shared $\epsilon\sim\mathcal{N}(0,\mathbf{I})$, independently of the noises used to generate the candidates.
We form $A_\tau^{\prime,s}=(1-\tau)A^s+\tau\epsilon$ for $s\in\{w,l\}$.
Using demonstration MSE to rank candidates and retaining pairs with $M_{\mathrm{pair}}=1$ gives
\begin{equation}
    \mathcal{L}_{\mathrm{DemoDPO}}
    =
    -\mathbb{E}_{\mathcal{D},\tau,\epsilon}
    \left[
        \log\sigma\left(
            \beta\left[
                s_\theta(A^w\mid x_t;\tau,\epsilon)
                -
                s_\theta(A^l\mid x_t;\tau,\epsilon)
            \right]
        \right)
        \,\middle|\,
        M_{\mathrm{pair}}=1
    \right].
    \label{eq:flow_dpo_final}
\end{equation}
This is Equation~\ref{eq:demodpo_objective}: the sampled preference loss uses velocity errors, a frozen reference policy, and demonstration rankings, without evaluating action likelihoods or complete generation paths.

\subsection{Training Algorithm}

Algorithm~\ref{alg:demodpo} summarizes one DemoDPO update.
Sampling and candidate selection are performed independently for each training example.
\texttt{mse} averages over valid action entries while retaining the batch dimension and, for candidate ranking, the candidate dimension.
\texttt{M} denotes $M_{\mathrm{pair}}$. \texttt{sample\_time()} draws $\tau\sim\operatorname{Unif}[0,1]$ independently for each example, with the same $\tau$ and $\epsilon$ shared by both candidates.

\begin{algorithm}[H]
\caption{DemoDPO with Demonstration-Guided Group Ranking}
\label{alg:demodpo}
\begin{lstlisting}[
  style=demodpo,
  basicstyle=\fontencoding{T1}\fontfamily{lmtt}\fontsize{10pt}{11pt}\selectfont,
  keywordstyle=\color{black},
  deletekeywords={sum,min},
  emph={fn,ref,stopgrad,sample,mse,argmin,argmax,sample_time,randn_like,logsigmoid,sum,clamp},
  literate=*
    {*}{{{\color{codesign}*}}}{1}
    {-}{{{\color{codesign}-}}}{1}
    {+}{{{\color{codesign}+}}}{1}
    {/}{{{\color{codesign}/}}}{1}
    {>=}{{{\color{codesign}>=}}}{2}
]
# fn, ref: velocity predictor and frozen reference
# x, A_demo: condition and demonstrated actions
# G, pair_epsilon, beta: candidate count, pair threshold, DPO coefficient

A = stopgrad(sample(ref, x, G))         # G candidates per condition
d = mse(A, A_demo)                      # demonstration error
w, l = argmin(d), argmax(d)
M = d[l] - d[w] >= pair_epsilon         # retain distinct pairs
A_w, A_l = A[w], A[l]

tau = sample_time()
eps = randn_like(A_demo)                # shared by both candidates
z_w = (1 - tau) * A_w + tau * eps
z_l = (1 - tau) * A_l + tau * eps
v_w, v_l = eps - A_w, eps - A_l

err_w = mse(fn(z_w, x, tau), v_w)
err_l = mse(fn(z_l, x, tau), v_l)
ref_w = mse(ref(z_w, x, tau), v_w)
ref_l = mse(ref(z_l, x, tau), v_l)

s_w = -0.5 * (err_w - stopgrad(ref_w))  # relative preference score
s_l = -0.5 * (err_l - stopgrad(ref_l))
loss = -logsigmoid(beta * (s_w - s_l))
loss = (loss * M).sum() / M.sum().clamp(min=1)
\end{lstlisting}
\end{algorithm}

\section{Implementation Details}
\label{app:implementation_details}
\label{app:additional_experimental_results}

We provide the training sample construction, frozen encoder comparison, evaluation protocols, visual representation analysis, and per task RoboTwin results.

\subsection{Training Sample Construction}
\label{app:training_sample_construction}

Following Fast-WAM~\citep{fast-wam}, we concatenate images from the head and two wrist cameras into a single image before feeding it to the vision encoder.
Because our study aims to isolate the effect of visual representations, we instantiate the task context $c_t$ as a discrete task ID and use the same conditioning across all controlled comparisons, rather than natural language instructions.
At control step $t$, each training sample contains the current visual observation $X_t$ and an action chunk $A_t=(a_t,\ldots,a_{t+H_{\mathrm a}-1})$, where $a_t\in\mathbb{R}^{D_a}$ and $D_a$ is the action dimension.
During training, we randomly sample one future offset $\delta\in\{1,\ldots,H_{\mathrm a}\}$ and use the corresponding observation $X_{t+\delta}$ for future embedding prediction.
Thus, we use $N=1$ future target per training sample.
Sampling offsets across training provides supervision over the prediction horizon while limiting the number of future query tokens.
Empirically, increasing $N$ provides limited performance gains while introducing more future query tokens and slowing training.
All LeWAM variants are trained only on the RoboTwin demonstrations described in Section~\ref{ssec:exp_setup}, while retaining the original visual pretraining of the frozen I-JEPA encoder.

\subsection{Frozen Encoder Comparison}
\label{app:frozen_encoder_ablation_details}

The goal of this comparison is to select a frozen visual encoder for LeWAM under a fixed downstream action predictor.
We compare representative checkpoints from each encoder family, with different architectures and parameter counts.
We first establish two representative reference points: the Wan2.2 VAE, which provides the compressed latent used by WAMs built on VDMs, and SigLIP~\citep{siglip}, a widely adopted visual encoder in VLAs.
We then compare MAE-Large~\citep{mae}, DINOv3-Large, and V-JEPA2-Large~\citep{vjepa2}.
Based on this comparison, we further evaluate V-JEPA2-Huge and I-JEPA-Huge~\citep{ijepa}.

For each encoder, we freeze all pretrained parameters and independently train the same lightweight DiT-style action predictor on RoboTwin.
Future embedding prediction is disabled so that the experiment evaluates each representation solely through action generation.
The action predictor architecture, current observations from multiple views, task conditioning, action horizon, optimizer, batch size, and training schedule are kept fixed across encoders.
Only the encoder-specific input projection and the downstream action predictor are trained.
\label{app:model_architecture}
All encoder variants use a 12-layer action predictor with a hidden size of 1280, and each encoder output is mapped to this width using a learned linear projection. After selecting I-JEPA-H/14, we use its frozen 0.6B-parameter encoder in LeWAM with an input resolution of $224 \times 224$ and a patch size of $14 \times 14$, producing 256 image tokens. The predictor and AdaFuse contain 402M trainable parameters, giving LeWAM 1.0B parameters in total.
Following DiT~\citep{DiT}, we use AdaLN-Zero to condition the predictor on noise timestep and task embeddings. Noisy action tokens use the sampled noise timestep, while visual tokens, clean action tokens, and future queries use the time embedding evaluated at zero. Before fusion, a shared LayerNorm independently normalizes each encoder layer's tokens across channels.

Each observation is processed with the input range, spatial resolution, and normalization for each channel expected by its pretrained encoder.
For consistency with the other encoders, we use only the current image, repeated 16 times following V-JEPA2's official single-image protocol~\citep{vjepa2}.
The complete action generation and probing results are reported in Table~\ref{tab:frozen_encoder_comparison} and discussed in RQ1.

\subsection{DemoDPO Ablation Evaluation Protocol}
\label{app:demodpo_ablation_details}

The DemoDPO hyperparameter ablations in Figure~\ref{fig:demodpo_ablations} use a fixed evaluation subset of eight challenging RoboTwin tasks to reduce evaluation cost.
All policies are trained on the complete RoboTwin 2.0 training set covering 50 tasks, comprising 2,500 clean and 25,000 randomized demonstrations, as described in Section~\ref{ssec:exp_setup}.
The subset is used only for evaluation and is not a training split by task.
For each task, we evaluate 100 clean and 100 randomized cases, resulting in 1,600 evaluation cases in total.
In alphabetical order, the task list is \texttt{blocks\_ranking\_size}, \texttt{hanging\_mug}, \texttt{open\_microwave}, \texttt{pick\_diverse\_bottles}, \texttt{place\_a2b\_right}, \texttt{place\_can\_basket}, \texttt{stamp\_seal}, and \texttt{turn\_switch}.

Under this evaluation protocol, the sweeps in Figure~\ref{fig:demodpo_ablations}, from \captiona{} to \captionc{}, vary $\varepsilon_{\mathrm{pair}}$ with $\beta=2000$ and $G=4$, $\beta$ with $\varepsilon_{\mathrm{pair}}=10^{-3}$ and $G=4$, and $G$ with $\varepsilon_{\mathrm{pair}}=10^{-3}$ and $\beta=5000$, respectively.

\subsection{Inference Efficiency}
\label{app:inference_efficiency}

We benchmark all models on an NVIDIA H800 with batch size 1 and 10 denoising steps.
Latency includes visual encoding, language encoding and KV-cache prefilling where applicable, and complete action denoising, with images re-encoded for every request.
CPU image preprocessing, tokenization, CPU--GPU transfers, robot execution, and one-time CUDA Graph capture are excluded.
We report peak PyTorch-allocated GPU memory with CUDA Graph enabled, including model parameters and graph-held tensors but excluding capture-phase peaks.

Table~\ref{tab:inference_efficiency} shows that LeWAM achieves the lowest latency with CUDA Graph disabled or enabled and the lowest peak memory among the evaluated models.

\begin{table}[htbp]
  \centering
  \small
  \caption{GPU inference efficiency. Peak allocated memory is measured with CUDA Graph enabled.}
  \label{tab:inference_efficiency}
  \setlength{\tabcolsep}{5pt}
  \renewcommand{\arraystretch}{1.1}
  \begin{tabular}{lrrrr}
    \hline
    Model & \multicolumn{2}{c}{Latency (ms) $\downarrow$} & Peak memory & Action chunk \\
    \cline{2-3}
    & CUDA Graph off & CUDA Graph on & (GiB) $\downarrow$ & length \\
    \hline
    $\pi_{0.5}$ (PyTorch) & 162.45 & 53.14 & 7.037 & 50 \\
    LingBot-VLA-v2 & 711.31 & 126.53 & 12.089 & 50 \\
    Fast-WAM & 325.55 & 102.38 & 23.198 & 32 \\
    \rowcolor{blue!8} LeWAM & \textbf{103.53} & \textbf{31.90} & \textbf{1.988} & 32 \\
    \hline
  \end{tabular}
\end{table}

\subsection{Visual Representation Analysis}
\label{app:representation_analysis}
\label{app:policy_attention}
\label{app:representation_probing}

\paragraph{Policy attention visualization.}
Attention scores are averaged across the action tokens, all attention heads, and all transformer layers at each inference step.
Figure~\ref{fig:policy_attention_tasks} presents attention maps across several manipulation tasks using the same aggregation protocol as Figure~\ref{fig:policy_attention}.
In Blocks Ranking RGB and Blocks Ranking Size, attention shifts between blocks as the policy grasps and places them in sequence.
Prominent responses appear around the active block, gripper, and placement region in both clean and cluttered scenes.
These observations suggest that action generation emphasizes object locations and spatial relationships, motivating the depth and segmentation probes below.

\begin{figure}[htpb]
    \centering
    \IfFileExists{figures/policy_attention_tasks.pdf}{%
        \includegraphics[width=\linewidth]{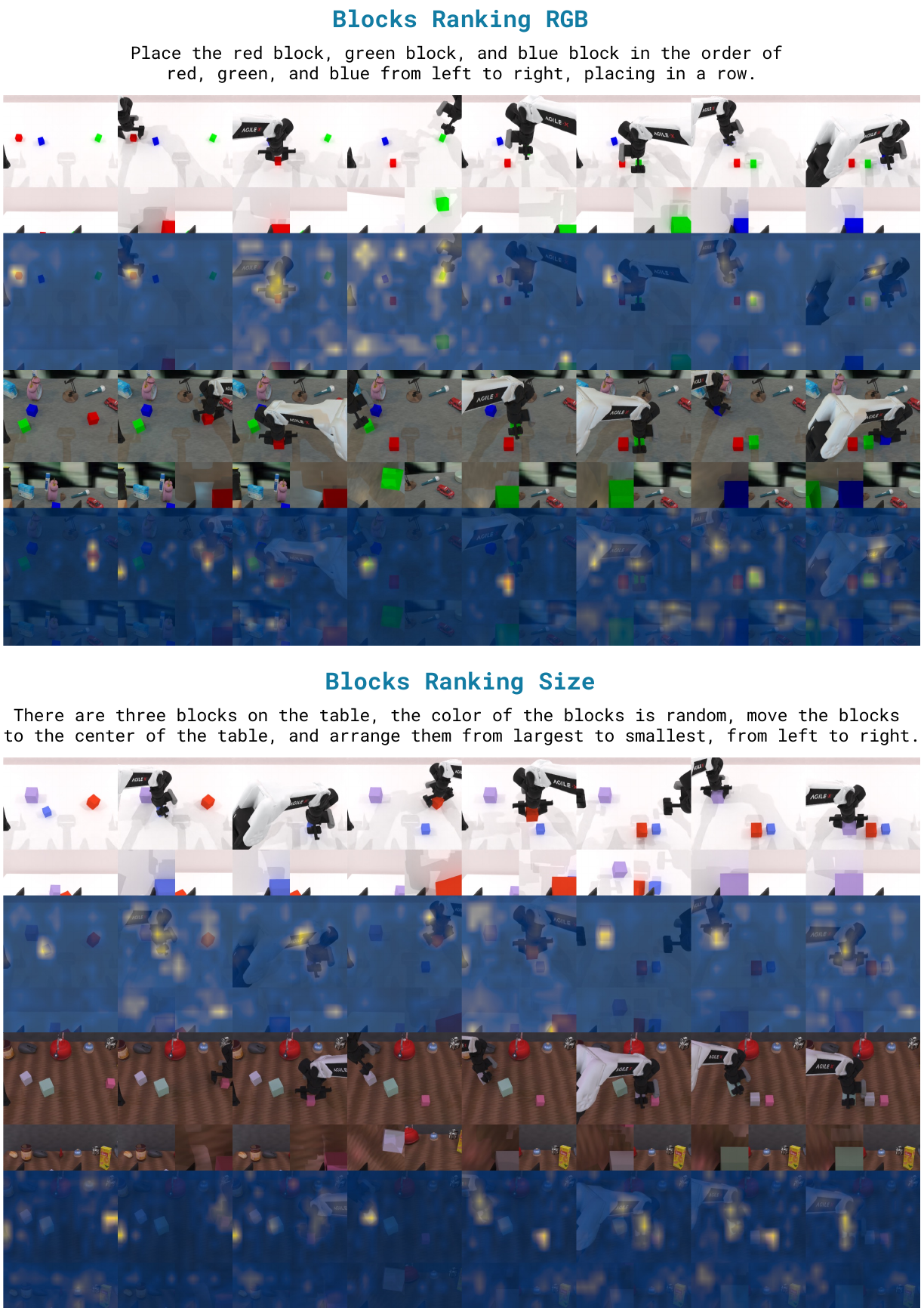}%
    }{%
        \fbox{\parbox[c][0.22\textheight][c]{0.96\linewidth}{\centering Policy attention visualization}}%
    }
    \caption{Policy attention visualizations under clean and randomized settings.
    Attention follows the active objects and interaction regions as the manipulation progresses.}
    \label{fig:policy_attention_tasks}
\end{figure}

Prior work probes pretrained visual encoders~\citep{zhu2023understanding,cae,ijepa,vjepa21}.
We analyze the information accessible from each frozen visual representation using depth and segmentation probes.
We randomly sample one timestep from each RoboTwin 2.0 training episode, obtaining 27,500 samples with synchronized front, left, and right views.
The samples are divided into training, validation, and test sets using an $8{:}1{:}1$ ratio.
All visual encoders remain frozen.
For AdaFuse, probing uses frozen fusion weights learned by LeWAM, while action success in Table~\ref{tab:frozen_encoder_comparison} uses action-only training.
Each probe applies LayerNorm, a linear projection to a width of 1280, and GELU to the frozen visual tokens.
The depth and segmentation mask heads operate on the spatial tokens, while the mask confidence head applies mean pooling before linear prediction.
Only the projection and probe heads are trained.

For depth probing, DA3MONO-LARGE~\citep{DA3} generates a relative depth target independently for each view.
The targets are standardized within each view and arranged in the same spatial layout as the visual input.
The depth probe is trained with mean squared error over valid pixels.
We evaluate depth prediction using scale and shift invariant root mean squared error (SSI-RMSE, shown as RMSE in Table~\ref{tab:frozen_encoder_comparison}) and Spearman rank correlation (Corr.) between the predicted and teacher depth values.

SAM3~\citep{SAM3} generates instance masks and confidence targets independently for each view.
Predicted masks are assigned to the teacher masks using Hungarian matching.
The training objective combines binary cross entropy (BCE) for mask and confidence prediction with Dice loss.
We report average precision at intersection over union (IoU) thresholds of 0.50 (AP@50) and 0.75 (AP@75).
A predicted mask is correct when its IoU with a teacher mask exceeds the threshold.

Table~\ref{tab:frozen_encoder_comparison} reports the probing results together with action success.
Among the individual encoders, I-JEPA-Huge performs best on both depth and segmentation.
The AdaFuse representation further improves all depth and segmentation metrics relative to the final I-JEPA layer.
Figure~\ref{fig:representation_probing_qualitative} presents representative depth and segmentation predictions.

\begin{figure}[p]
    \centering
    \includegraphics[width=\linewidth,height=0.92\textheight,keepaspectratio]{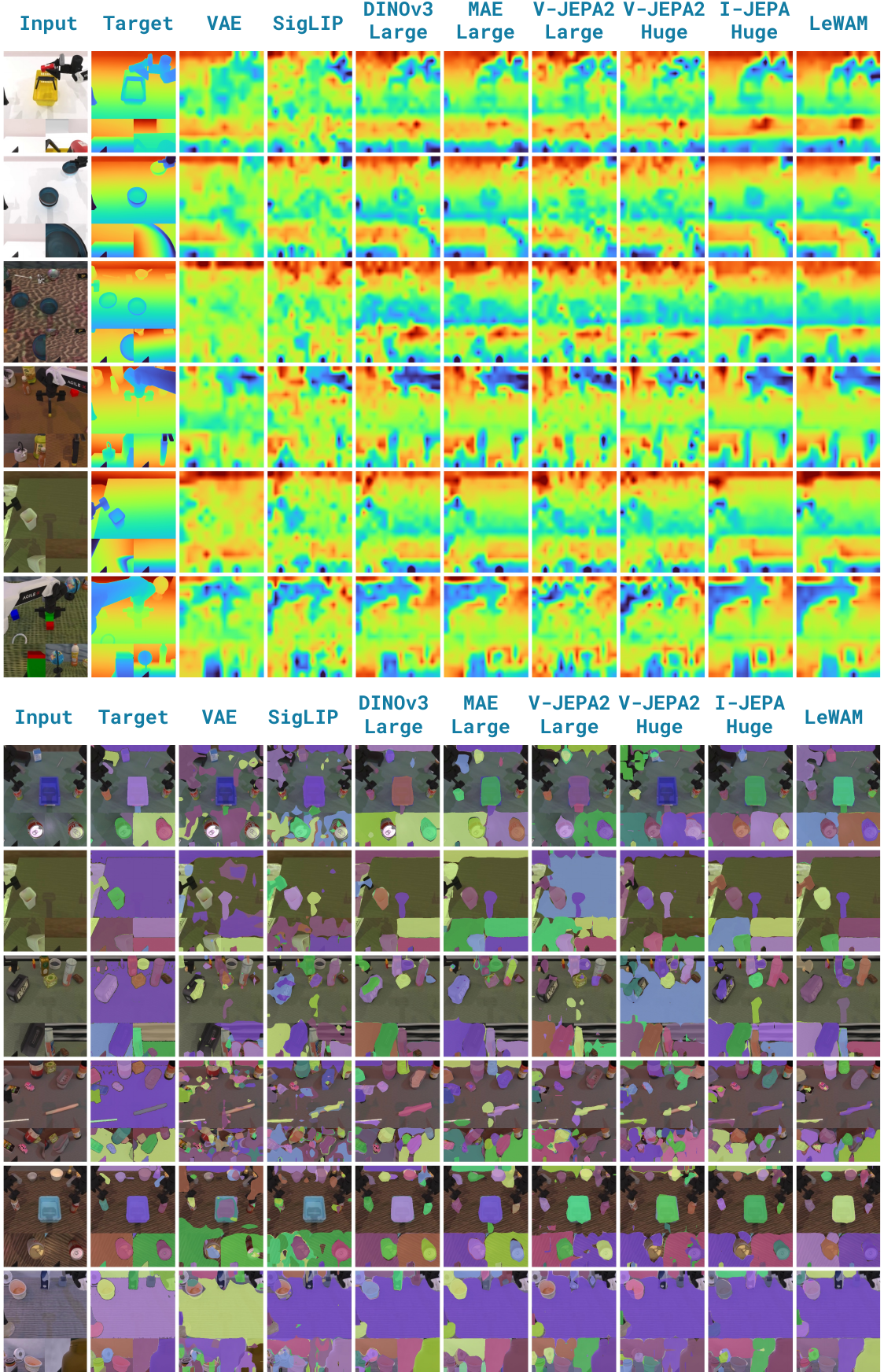}
    \caption{Depth and segmentation probing across visual representations.
    LeWAM denotes I-JEPA with learned AdaFuse weights.
    Depth maps share a fixed color scale, while segmentation colors distinguish local mask instances rather than semantic categories.
    }
    \label{fig:representation_probing_qualitative}
\end{figure}

\subsection{RoboTwin Per-Task Results}
\label{app:robotwin_per_task_results}

Table~\ref{tab:robotwin_per_task} reports success rates on all 50 RoboTwin 2.0 tasks under clean and randomized settings.
In Table~\ref{tab:robotwin_per_task}, results for GigaWorld-Policy and LaWAM are taken from their original papers~\citep{gigaworld-policy,lawam}, while results for all other external baselines are taken from~\citet{fast-wam}. 
LeWAM achieves an overall success rate of 92.28\%, with 93.14\% in clean settings and 91.42\% under randomization, matching the performance of leading VLA and WAM baselines.
\begin{table}[htpb]
    \centering
    \tiny
    \caption{Per-task success rates (\%) on the 50 RoboTwin 2.0 tasks under clean and randomized evaluation settings. Best results are shown in bold.}
    \label{tab:robotwin_per_task}
    \setlength{\tabcolsep}{2pt}
    \resizebox{\textwidth}{!}{%
        \begin{tabular}{lcc|cc|cc|cc|cc|cc|cc}
            \toprule
            \raisebox{0.55em}{Task} & \multicolumn{2}{c}{\raisebox{0.55em}{$\pi_{0.5}$}} & \multicolumn{2}{c}{\shortstack{GigaWorld-\\Policy}} & \multicolumn{2}{c}{\raisebox{0.55em}{Motus}} & \multicolumn{2}{c}{\raisebox{0.55em}{LaWAM}} & \multicolumn{2}{c}{\raisebox{0.55em}{Fast-WAM}} & \multicolumn{2}{c}{\raisebox{0.55em}{LingBot-VA}} & \multicolumn{2}{c}{\raisebox{0.55em}{LeWAM}} \\
            \cmidrule(lr){2-3}\cmidrule(lr){4-5}\cmidrule(lr){6-7}\cmidrule(lr){8-9}\cmidrule(lr){10-11}\cmidrule(lr){12-13}\cmidrule(lr){14-15}
                 & Clean & Rand. & Clean & Rand. & Clean & Rand. & Clean & Rand. & Clean & Rand. & Clean & Rand. & Clean & Rand. \\
            \midrule
            Adjust Bottle & \textbf{100} & 99 & \textbf{100} & \textbf{100} & 89 & 93 & \textbf{100} & \textbf{100} & \textbf{100} & \textbf{100} & 90 & 94 & \textbf{100} & \textbf{100} \\
            Beat Block Hammer & 96 & 93 & 86 & 86 & 95 & 88 & 90 & 93 & \textbf{99} & 97 & 96 & \textbf{98} & \textbf{99} & 97 \\
            Blocks Ranking RGB & 92 & 85 & 92 & 96 & 99 & 97 & 97 & \textbf{100} & \textbf{100} & \textbf{100} & 99 & 98 & \textbf{100} & 99 \\
            Blocks Ranking Size & 49 & 26 & 44 & 48 & 75 & 63 & 93 & 89 & \textbf{94} & \textbf{98} & \textbf{94} & 96 & 81 & 80 \\
            Click Alarmclock & 98 & 89 & \textbf{100} & \textbf{100} & \textbf{100} & \textbf{100} & \textbf{100} & \textbf{100} & \textbf{100} & \textbf{100} & 99 & \textbf{100} & \textbf{100} & 98 \\
            Click Bell & 99 & 66 & \textbf{100} & \textbf{100} & \textbf{100} & \textbf{100} & \textbf{100} & \textbf{100} & \textbf{100} & \textbf{100} & \textbf{100} & \textbf{100} & \textbf{100} & \textbf{100} \\
            Dump Bin Bigbin & 92 & 97 & 92 & \textbf{100} & 95 & 91 & \textbf{97} & 95 & \textbf{97} & 96 & 89 & 96 & \textbf{97} & 98 \\
            Grab Roller & \textbf{100} & \textbf{100} & \textbf{100} & \textbf{100} & \textbf{100} & \textbf{100} & \textbf{100} & \textbf{100} & \textbf{100} & \textbf{100} & \textbf{100} & \textbf{100} & \textbf{100} & \textbf{100} \\
            Handover Block & 66 & 57 & 80 & 80 & 86 & 73 & 96 & 87 & 95 & 81 & 99 & 78 & \textbf{100} & \textbf{93} \\
            Handover Mic & 98 & 97 & 72 & 72 & 78 & 63 & 93 & 98 & \textbf{99} & \textbf{100} & 94 & 96 & \textbf{99} & 99 \\
            Hanging Mug & 18 & 17 & 16 & 12 & 38 & 38 & 51 & 43 & 58 & \textbf{62} & 40 & 28 & \textbf{65} & 53 \\
            Lift Pot & 96 & 85 & 98 & 98 & 96 & 99 & \textbf{100} & 99 & \textbf{100} & \textbf{100} & \textbf{100} & 99 & \textbf{100} & \textbf{100} \\
            Move Can Pot & 51 & 55 & 76 & 78 & 34 & 74 & \textbf{98} & 93 & 90 & 88 & 94 & \textbf{97} & 93 & 93 \\
            Move Pillbottle Pad & 84 & 61 & 90 & 90 & 93 & 96 & 97 & 90 & \textbf{100} & \textbf{99} & 99 & \textbf{99} & 97 & 96 \\
            Move Playingcard Away & 96 & 84 & 78 & 72 & \textbf{100} & 96 & \textbf{100} & \textbf{100} & \textbf{100} & \textbf{100} & \textbf{100} & 99 & \textbf{100} & \textbf{100} \\
            Move Stapler Pad & 56 & 42 & 92 & 82 & 83 & 85 & \textbf{94} & \textbf{87} & 77 & 64 & 91 & 79 & 93 & 85 \\
            Open Laptop & 90 & 96 & 96 & 98 & 95 & 91 & \textbf{100} & \textbf{100} & 98 & \textbf{100} & 92 & 94 & 99 & \textbf{100} \\
            Open Microwave & 34 & 77 & 74 & 66 & \textbf{95} & \textbf{91} & 41 & 43 & 62 & 45 & 82 & 86 & 87 & 86 \\
            Pick Diverse Bottles & 81 & 71 & 82 & 70 & 90 & \textbf{91} & \textbf{91} & 88 & 80 & 85 & 89 & 82 & 87 & 81 \\
            Pick Dual Bottles & 93 & 63 & 86 & 86 & 96 & 90 & \textbf{100} & 95 & \textbf{100} & 96 & \textbf{100} & \textbf{99} & 99 & 94 \\
            Place A2B Left & 87 & 82 & 94 & 88 & 88 & 79 & \textbf{98} & 91 & 95 & \textbf{93} & 97 & \textbf{93} & 81 & 76 \\
            Place A2B Right & 87 & 84 & 90 & 92 & 91 & 87 & 89 & 94 & 93 & \textbf{99} & \textbf{97} & 95 & 85 & 80 \\
            Place Bread Basket & 77 & 64 & 82 & 82 & 91 & 94 & 92 & 85 & 91 & 93 & \textbf{97} & \textbf{95} & \textbf{97} & \textbf{95} \\
            Place Bread Skillet & 85 & 66 & 94 & 90 & 86 & 83 & 90 & 83 & 90 & \textbf{93} & \textbf{95} & 90 & 94 & 85 \\
            Place Burger Fries & 94 & 87 & 98 & 96 & 98 & 98 & 93 & 96 & 96 & \textbf{99} & 97 & 95 & \textbf{99} & 98 \\
            Place Can Basket & 62 & 62 & 78 & 74 & 81 & 76 & \textbf{92} & 65 & 71 & 69 & 81 & \textbf{84} & 82 & 67 \\
            Place Cans Plasticbox & 94 & 84 & \textbf{100} & \textbf{100} & 98 & 94 & \textbf{100} & 95 & 99 & 96 & \textbf{100} & 99 & \textbf{100} & 99 \\
            Place Container Plate & 99 & 95 & 98 & 96 & 98 & 99 & \textbf{100} & \textbf{100} & 96 & \textbf{100} & 99 & 97 & \textbf{100} & \textbf{100} \\
            Place Dual Shoes & 75 & 75 & 96 & 84 & 93 & 87 & \textbf{98} & \textbf{94} & 94 & 88 & 94 & 89 & 86 & 90 \\
            Place Empty Cup & \textbf{100} & 99 & 90 & 90 & 99 & 98 & 99 & \textbf{100} & \textbf{100} & \textbf{100} & \textbf{100} & \textbf{100} & \textbf{100} & \textbf{100} \\
            Place Fan & 87 & 85 & 92 & 94 & 91 & 87 & 92 & 93 & 96 & \textbf{96} & \textbf{99} & 93 & 96 & 94 \\
            Place Mouse Pad & 60 & 39 & 88 & 90 & 66 & 68 & 91 & 84 & 83 & 89 & \textbf{93} & \textbf{96} & \textbf{93} & 85 \\
            Place Object Basket & 80 & 76 & 90 & \textbf{92} & 81 & 87 & \textbf{92} & 90 & 89 & 88 & 91 & 88 & 89 & 90 \\
            Place Object Scale & 86 & 80 & 88 & 80 & 88 & 85 & 95 & 88 & 90 & \textbf{97} & \textbf{96} & 95 & 89 & 95 \\
            Place Object Stand & 91 & 85 & \textbf{100} & \textbf{98} & 98 & 97 & 92 & 93 & 90 & 94 & 99 & 96 & 98 & 96 \\
            Place Phone Stand & 81 & 81 & 82 & 72 & 87 & 86 & 93 & 94 & \textbf{97} & \textbf{99} & \textbf{97} & 97 & \textbf{97} & 98 \\
            Place Shoe & 92 & 93 & 98 & 96 & 99 & 97 & \textbf{100} & \textbf{100} & 96 & 99 & 98 & 98 & 99 & 97 \\
            Press Stapler & 87 & 83 & 96 & 96 & 93 & \textbf{98} & \textbf{98} & 97 & 90 & 97 & 85 & 82 & 83 & 85 \\
            Put Bottles Dustbin & 84 & 79 & 72 & 70 & 81 & 79 & 94 & \textbf{92} & \textbf{95} & 90 & 87 & 91 & 94 & 90 \\
            Put Object Cabinet & 80 & 79 & 74 & 74 & 88 & 71 & 90 & 82 & \textbf{94} & \textbf{89} & 85 & 87 & 91 & 83 \\
            Rotate QRcode & 89 & 87 & 90 & 84 & 89 & 73 & 94 & 89 & 93 & 89 & \textbf{96} & 91 & 93 & \textbf{92} \\
            Scan Object & 72 & 65 & 60 & 64 & 67 & 66 & \textbf{96} & 90 & 89 & \textbf{92} & \textbf{96} & 91 & 89 & 90 \\
            Shake Bottle & 99 & 97 & \textbf{100} & 98 & \textbf{100} & 97 & \textbf{100} & \textbf{100} & \textbf{100} & \textbf{100} & \textbf{100} & 97 & \textbf{100} & 99 \\
            Shake Bottle Horizontally & 99 & 99 & \textbf{100} & \textbf{100} & \textbf{100} & 98 & \textbf{100} & \textbf{100} & \textbf{100} & \textbf{100} & \textbf{100} & 99 & \textbf{100} & 98 \\
            Stack Blocks Three & 91 & 76 & 70 & 78 & 91 & 95 & 90 & 75 & 95 & 97 & 99 & \textbf{98} & \textbf{100} & 96 \\
            Stack Blocks Two & 97 & \textbf{100} & \textbf{100} & 94 & \textbf{100} & 98 & \textbf{100} & 97 & \textbf{100} & \textbf{100} & \textbf{100} & 98 & \textbf{100} & \textbf{100} \\
            Stack Bowls Three & 77 & 71 & 70 & 72 & 79 & \textbf{87} & \textbf{90} & 80 & 80 & 81 & 86 & 83 & 85 & 86 \\
            Stack Bowls Two & 95 & 96 & 96 & 92 & 98 & 98 & \textbf{100} & \textbf{99} & 92 & 98 & 94 & 98 & 99 & 95 \\
            Stamp Seal & 79 & 55 & \textbf{96} & \textbf{98} & 93 & 92 & 89 & 88 & 90 & 94 & \textbf{96} & 97 & 82 & 92 \\
            Turn Switch & 62 & 54 & 82 & \textbf{84} & \textbf{84} & 78 & 47 & 56 & 61 & 59 & 44 & 45 & 60 & 68 \\
            \midrule
            \textbf{Average} & 82.74 & 76.76 & 86.36 & 85.04 & 88.66 & 87.02 & 92.64 & 89.80 & 91.88 & \textbf{91.78} & 92.90 & 91.50 & \textbf{93.14} & 91.42 \\
            \textbf{Overall} & \multicolumn{2}{c|}{79.75} & \multicolumn{2}{c|}{85.70} & \multicolumn{2}{c|}{87.84} & \multicolumn{2}{c|}{91.22} & \multicolumn{2}{c|}{91.83} & \multicolumn{2}{c|}{92.20} & \multicolumn{2}{c}{\textbf{92.28}} \\
            \bottomrule
        \end{tabular}
    }
\end{table}

\subsection{LeWAM Per-Task Ablations}
\label{app:lewam_ablation_results}

Table~\ref{tab:lewam_ablation_per_task} reports per-task results for the LeWAM ablations.
Direct future embedding prediction provides a modest overall gain over the Action Only baseline, improving the success rate from 86.17\% to 87.02\%, while the flow matching variant reaches 86.21\%.
Adding AdaFuse produces a larger improvement to 90.69\%.
DemoDPO further raises the success rate to 92.28\%. Compared with the AdaFuse variant, it improves 25 clean and 24 randomized per-task results.
Relative to the Action Only baseline, the complete LeWAM improves 36 clean and 41 randomized per-task scores, showing that the gain is distributed across many tasks rather than concentrated in a few cases.
\begin{table}[htpb]
    \centering
    \tiny
    \caption{Per-task LeWAM ablations on RoboTwin 2.0. LeWAM uses I-JEPA-H/14 with 402M trainable parameters. WM (FM) uses flow matching for future embedding prediction. The remaining variants progressively add WM, AdaFuse, and DemoDPO. Best results are shown in bold.}
    \label{tab:lewam_ablation_per_task}
    \setlength{\tabcolsep}{3.5pt}
    \renewcommand{\arraystretch}{0.8}
    \resizebox{\textwidth}{!}{%
        \begin{tabular}{lcc|cc|cc|cc|cc}
            \toprule
            \raisebox{0.55em}{Task} & \multicolumn{2}{c}{\shortstack{Action\\Only}} & \multicolumn{2}{c}{\shortstack{+ WM\\(FM)}} & \multicolumn{2}{c}{\shortstack{+ WM}} & \multicolumn{2}{c}{\shortstack{+ WM\\+ AdaFuse}} & \multicolumn{2}{c}{\raisebox{0.55em}{LeWAM}} \\
            \cmidrule(lr){2-3}\cmidrule(lr){4-5}\cmidrule(lr){6-7}\cmidrule(lr){8-9}\cmidrule(lr){10-11}
                 & Clean & Rand. & Clean & Rand. & Clean & Rand. & Clean & Rand. & Clean & Rand. \\
            \midrule
            Adjust Bottle & 99 & \textbf{100} & \textbf{100} & 99 & \textbf{100} & 99 & \textbf{100} & \textbf{100} & \textbf{100} & \textbf{100} \\
            Beat Block Hammer & 98 & 92 & 98 & 94 & 96 & \textbf{97} & \textbf{99} & 94 & \textbf{99} & \textbf{97} \\
            Blocks Ranking RGB & 98 & 96 & 99 & 98 & 99 & 98 & \textbf{100} & \textbf{99} & \textbf{100} & \textbf{99} \\
            Blocks Ranking Size & 78 & 63 & 77 & 76 & 77 & 68 & 79 & \textbf{80} & \textbf{81} & \textbf{80} \\
            Click Alarmclock & \textbf{100} & 96 & \textbf{100} & 98 & \textbf{100} & \textbf{99} & \textbf{100} & 98 & \textbf{100} & 98 \\
            Click Bell & \textbf{100} & 99 & \textbf{100} & \textbf{100} & \textbf{100} & \textbf{100} & \textbf{100} & \textbf{100} & \textbf{100} & \textbf{100} \\
            Dump Bin Bigbin & 95 & 95 & \textbf{100} & 97 & 94 & 97 & 97 & 94 & 97 & \textbf{98} \\
            Grab Roller & 99 & 99 & 99 & \textbf{100} & \textbf{100} & 99 & \textbf{100} & \textbf{100} & \textbf{100} & \textbf{100} \\
            Handover Block & 96 & 82 & 78 & 78 & 93 & 72 & 96 & 81 & \textbf{100} & \textbf{93} \\
            Handover Mic & \textbf{99} & 98 & 98 & 98 & \textbf{99} & 94 & 98 & \textbf{99} & \textbf{99} & \textbf{99} \\
            Hanging Mug & \textbf{65} & 46 & 24 & 12 & 47 & 38 & 46 & \textbf{54} & \textbf{65} & 53 \\
            Lift Pot & \textbf{100} & 69 & \textbf{100} & 68 & 99 & 72 & \textbf{100} & \textbf{100} & \textbf{100} & \textbf{100} \\
            Move Can Pot & 79 & 75 & 88 & 86 & \textbf{98} & \textbf{94} & 89 & 92 & 93 & 93 \\
            Move Pillbottle Pad & 84 & 89 & 94 & 92 & 96 & 94 & \textbf{97} & \textbf{96} & \textbf{97} & \textbf{96} \\
            Move Playingcard Away & 99 & 99 & \textbf{100} & 98 & \textbf{100} & \textbf{100} & \textbf{100} & 99 & \textbf{100} & \textbf{100} \\
            Move Stapler Pad & 71 & 62 & 71 & 70 & 75 & 80 & \textbf{94} & \textbf{86} & 93 & 85 \\
            Open Laptop & 96 & \textbf{100} & 98 & 98 & 94 & \textbf{100} & 98 & \textbf{100} & \textbf{99} & \textbf{100} \\
            Open Microwave & 76 & 69 & \textbf{89} & \textbf{91} & 67 & 67 & 76 & 82 & 87 & 86 \\
            Pick Diverse Bottles & 72 & 75 & 66 & 68 & 73 & 66 & 83 & 78 & \textbf{87} & \textbf{81} \\
            Pick Dual Bottles & 98 & 91 & 94 & 90 & 91 & 92 & \textbf{99} & \textbf{94} & \textbf{99} & \textbf{94} \\
            Place A2B Left & 82 & 70 & 81 & 73 & \textbf{84} & \textbf{76} & 79 & 71 & 81 & \textbf{76} \\
            Place A2B Right & 76 & 63 & 81 & 65 & 79 & 67 & 83 & 79 & \textbf{85} & \textbf{80} \\
            Place Bread Basket & 91 & 90 & 90 & 86 & 88 & 93 & 92 & 92 & \textbf{97} & \textbf{95} \\
            Place Bread Skillet & 94 & 79 & 91 & 81 & 91 & 86 & \textbf{97} & \textbf{87} & 94 & 85 \\
            Place Burger Fries & 97 & 98 & 98 & 93 & 96 & 95 & \textbf{99} & \textbf{100} & \textbf{99} & 98 \\
            Place Can Basket & 46 & 53 & 73 & 61 & 73 & 61 & 75 & 62 & \textbf{82} & \textbf{67} \\
            Place Cans Plasticbox & 99 & 95 & \textbf{100} & \textbf{99} & 99 & 93 & \textbf{100} & 97 & \textbf{100} & \textbf{99} \\
            Place Container Plate & 99 & 94 & 97 & 94 & 99 & 93 & \textbf{100} & \textbf{100} & \textbf{100} & \textbf{100} \\
            Place Dual Shoes & 83 & 80 & \textbf{88} & 86 & 76 & 80 & 81 & 84 & 86 & \textbf{90} \\
            Place Empty Cup & \textbf{100} & \textbf{100} & \textbf{100} & 98 & \textbf{100} & 99 & \textbf{100} & \textbf{100} & \textbf{100} & \textbf{100} \\
            Place Fan & 87 & 86 & \textbf{97} & \textbf{94} & 96 & 89 & \textbf{97} & 90 & 96 & \textbf{94} \\
            Place Mouse Pad & 76 & 75 & 77 & 74 & 87 & \textbf{85} & 89 & 84 & \textbf{93} & \textbf{85} \\
            Place Object Basket & \textbf{94} & 83 & 79 & 77 & 80 & 74 & 90 & \textbf{90} & 89 & \textbf{90} \\
            Place Object Scale & 83 & 75 & 86 & 80 & 82 & 80 & 85 & 92 & \textbf{89} & \textbf{95} \\
            Place Object Stand & 85 & 93 & 92 & 90 & 92 & 90 & 94 & 95 & \textbf{98} & \textbf{96} \\
            Place Phone Stand & 96 & 84 & 94 & 81 & \textbf{97} & 82 & \textbf{97} & \textbf{98} & \textbf{97} & \textbf{98} \\
            Place Shoe & 93 & 93 & 93 & 92 & 94 & 94 & \textbf{99} & 94 & \textbf{99} & \textbf{97} \\
            Press Stapler & \textbf{83} & 79 & 80 & 75 & 80 & 83 & 79 & \textbf{85} & \textbf{83} & \textbf{85} \\
            Put Bottles Dustbin & 88 & 74 & 92 & 76 & 87 & 73 & \textbf{94} & 89 & \textbf{94} & \textbf{90} \\
            Put Object Cabinet & \textbf{92} & \textbf{84} & 81 & 81 & 81 & 80 & 90 & 77 & 91 & 83 \\
            Rotate QRcode & 94 & 90 & \textbf{96} & \textbf{93} & 84 & 90 & 93 & 87 & 93 & 92 \\
            Scan Object & 88 & 73 & 87 & 75 & 88 & 72 & 87 & 87 & \textbf{89} & \textbf{90} \\
            Shake Bottle & 99 & \textbf{100} & 99 & \textbf{100} & 99 & \textbf{100} & 99 & 98 & \textbf{100} & 99 \\
            Shake Bottle Horizontally & \textbf{100} & \textbf{100} & \textbf{100} & \textbf{100} & 99 & \textbf{100} & 99 & 99 & \textbf{100} & 98 \\
            Stack Blocks Three & 96 & 95 & 99 & 91 & \textbf{100} & \textbf{99} & 99 & 97 & \textbf{100} & 96 \\
            Stack Blocks Two & 99 & 99 & \textbf{100} & \textbf{100} & \textbf{100} & 98 & \textbf{100} & \textbf{100} & \textbf{100} & \textbf{100} \\
            Stack Bowls Three & 83 & 78 & 78 & 79 & 81 & 70 & 78 & \textbf{89} & \textbf{85} & 86 \\
            Stack Bowls Two & 97 & 95 & 97 & 90 & 98 & 95 & 98 & \textbf{97} & \textbf{99} & 95 \\
            Stamp Seal & 53 & 53 & 53 & 48 & 77 & 66 & 79 & 86 & \textbf{82} & \textbf{92} \\
            Turn Switch & \textbf{66} & 70 & 57 & 69 & 58 & 70 & 63 & \textbf{71} & 60 & 68 \\
            \midrule
            \textbf{Average} & 88.42 & 83.92 & 88.18 & 84.24 & 88.86 & 85.18 & 91.32 & 90.06 & \textbf{93.14} & \textbf{91.42} \\
            \textbf{Overall} & \multicolumn{2}{c|}{86.17} & \multicolumn{2}{c|}{86.21} & \multicolumn{2}{c|}{87.02} & \multicolumn{2}{c|}{90.69} & \multicolumn{2}{c}{\textbf{92.28}} \\
            \bottomrule
        \end{tabular}
    }
\end{table}

\section{Real-World Experiments}
\label{app:real_world_protocol}

\paragraph{Data collection.}
We conduct the real-world experiments on an AgileX dual arm platform consisting of two Piper manipulators.
Demonstrations are collected through leader-follower teleoperation.
The policy observes synchronized RGB images from a fixed front camera and cameras mounted on the left and right robot arms.
We collect 400 episodes for each of Stack Blocks, Fold Towel, and Arrange Flowers, resulting in 1,200 episodes and approximately 11 hours of data collected in house.
Both observations and actions are recorded at 30 Hz.
The initial object positions are varied across episodes.
All demonstrations use delta joint positions as the action representation.

\paragraph{Policy training.}
All evaluated methods are trained on the same demonstrations for 20,000 optimization steps.
The baselines are finetuned from their publicly available pretrained checkpoints.
We use AdamW with a learning rate of $5\times10^{-5}$ and a batch size of 256 for all methods.
Each policy predicts an action chunk of 50 steps.
We use the checkpoints after 20,000 supervised optimization steps for baseline evaluation.
For LeWAM, we further apply DemoDPO for 1,000 updates using the same hyperparameters as in Section~\ref{ssec:exp_setup}, including $G=4$ and $\beta=5000$, and evaluate the refined policy.

\paragraph{Robot evaluation.}
All policies are deployed on an NVIDIA A100 GPU with 40 GB of memory and executed at 30 Hz.
Before each rollout, both robot arms are returned to the same home configuration.
The objects are then placed in the initial configuration sampled for that rollout.
Each method is evaluated for 50 rollouts on each of the three tasks.
Each rollout is capped at three minutes.

\paragraph{Task progress.}
The three real-world tasks contain several sequential manipulation stages.
We measure task progress by scoring each recorded rollout according to completed milestones.

Let $s_{m,q,i}$ denote the score of rollout $i$ produced by method $m$ on task $q$, and let $M_q$ denote the maximum score for that task.
We report task progress as
\begin{equation}
    \operatorname{Progress}(m,q)
    =
    \frac{100}{N_q M_q}
    \sum_{i=1}^{N_q} s_{m,q,i},
    \label{eq:real_world_task_progress}
\end{equation}
where $N_q=50$.
Task progress is the average rollout score as a percentage of the maximum score.

Each milestone contributes one point and can be counted at most once in a rollout.
Task dependent penalties are applied after adding the milestone points.
The resulting score is clipped to the interval $[0,M_q]$.
All methods are evaluated from their recorded rollouts using the same scoring rules.

\paragraph{Stack Blocks ($M_q=7$).}
The objective is to form a three block stack with the red block at the bottom, the green block in the middle, and the blue block on top.
The rollout receives one point for each of the following milestones:
\begin{enumerate}
    \item The robot lifts the red block and moves it toward the target location.
    \item The red block is placed stably.
    \item The robot lifts the green block and moves it toward the red block.
    \item The green block is placed stably.
    \item The robot lifts the blue block and moves it toward the green block.
    \item The blue block is placed stably.
\end{enumerate}
The six manipulation milestones contribute at most six points.
The final configuration receives one additional point when the stack is stable and neatly aligned, and no additional point when it is stable but untidy.
Two points are deducted if the final structure collapses.
One point is deducted if the structure collapses during execution and the robot does not continue the task.

\paragraph{Fold Towel ($M_q=6$).}
The objective is to fold a towel and hang it on a rack.
The rollout receives one point for each of the following milestones:
\begin{enumerate}
    \item Most of the towel is unfolded.
    \item The towel is fully flattened.
    \item The towel is folded in half.
    \item The robot adjusts the orientation of the folded towel.
    \item The robot lifts the towel and moves it toward the rack.
    \item The towel is placed on the rack.
\end{enumerate}
No adjustment is applied when the towel remains stably and neatly placed on the rack.
One point is deducted if the towel falls from the rack or its final placement is irregular.

\paragraph{Arrange Flowers ($M_q=6$).}
The objective is to insert three flowers into a vase while keeping the vase upright.
For each flower, one point is awarded when the robot lifts it and moves it toward the vase, and another point is awarded when it is inserted successfully.
The three flowers therefore contribute at most six points.
If the vase is knocked over and the robot does not restore it, the rollout terminates and one point is deducted.

\subsection{Qualitative Comparisons}

DemoDPO raises task progress from 81.1\%, 58.3\%, and 50.7\% to 84.9\%, 60.7\%, and 53.0\% on Stack Blocks, Fold Towel, and Arrange Flowers, respectively, with example rollouts in Figure~\ref{fig:real_world_demodpo_comparison}.

\begin{figure}[!htb]
    \centering
    \includegraphics[width=\linewidth]{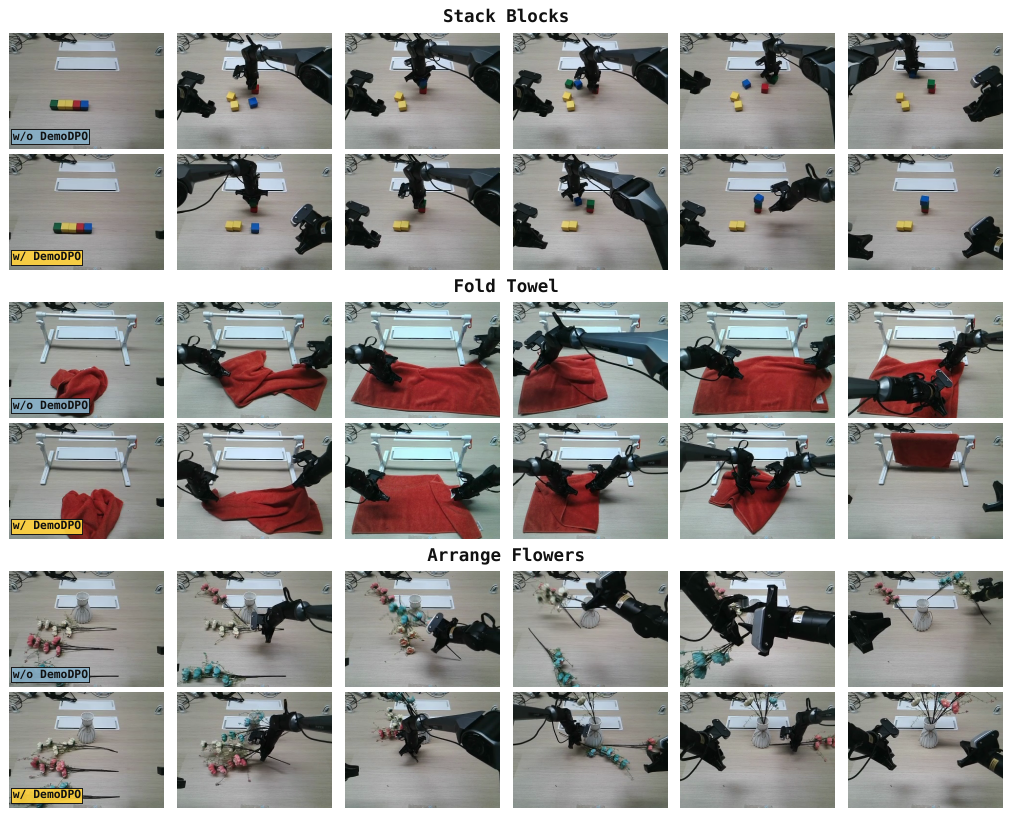}
    \vspace{-1.5em}
    \caption{Qualitative effect of DemoDPO on the three real-world tasks.
    For each task, the two rows compare LeWAM before and after DemoDPO refinement under the same initial configuration.}
    \label{fig:real_world_demodpo_comparison}
\end{figure}

Figure~\ref{fig:real_world_comparison_grid} compares representative rollouts for all four methods on the three real-world tasks.
These examples visualize the partial task completion captured by the task progress metric.

\begin{figure}[p]
    \centering
    \includegraphics[width=\linewidth,height=0.91\textheight,keepaspectratio]{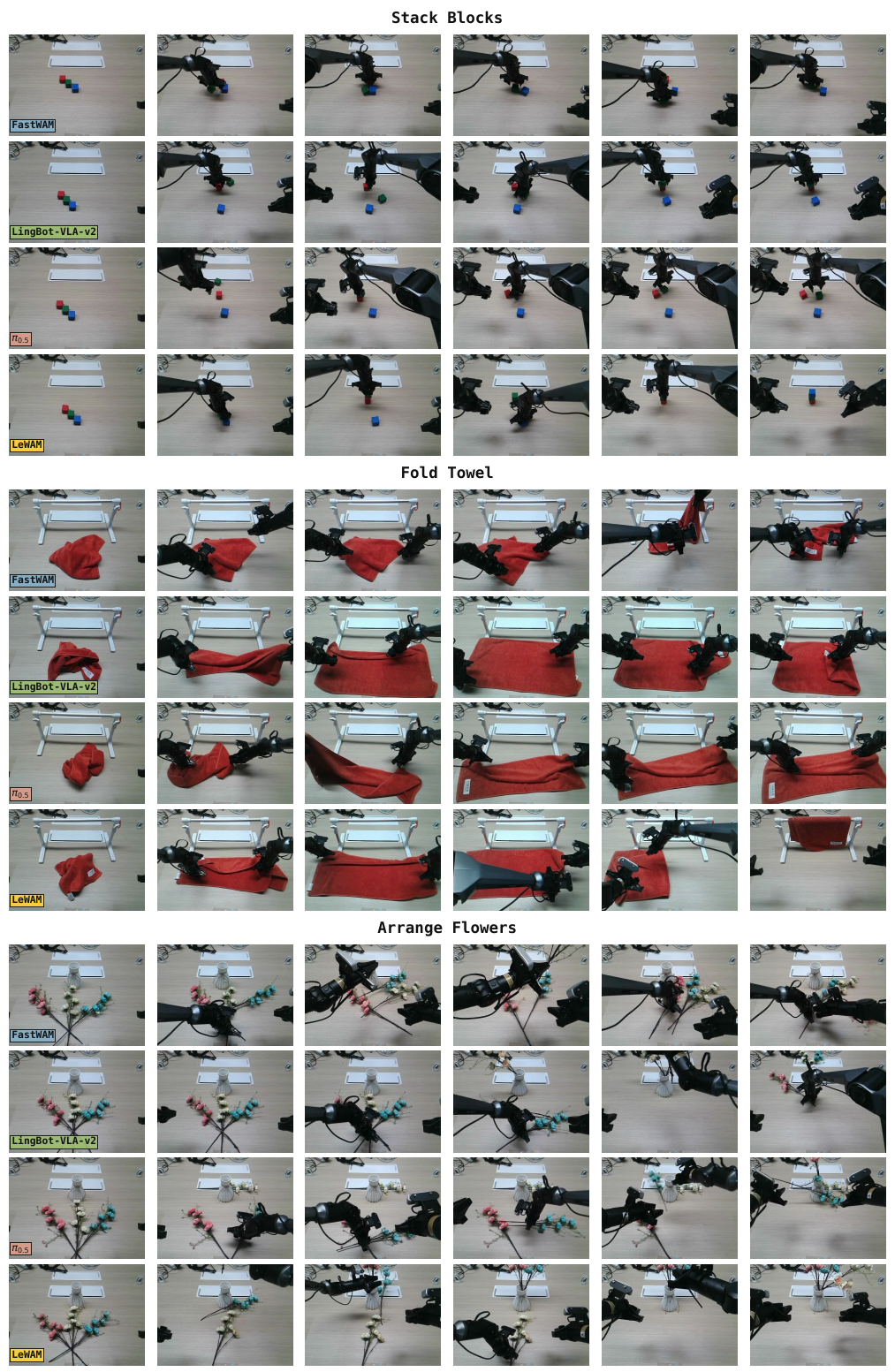}
    \caption{Representative real-world rollouts for Stack Blocks, Fold Towel, and Arrange Flowers.
    For each task, the sequences show task execution from the fixed front camera under the same initial configuration across methods.}
    \label{fig:real_world_comparison_grid}
\end{figure}

\end{document}